\documentclass[10pt,journal,compsoc]{IEEEtran}
\ifCLASSOPTIONcompsoc
\usepackage[nocompress]{cite}
\else
\usepackage{cite}
\fi
\ifCLASSINFOpdf
\else
\fi

	\usepackage[english]{babel}
	\usepackage{amsmath}
	\usepackage{amsfonts}
	\usepackage{array}
	\usepackage{graphicx}
	\usepackage{bm}
	\usepackage{fixmath}
	\usepackage{blindtext}
	\usepackage{enumitem}
	\usepackage{amsthm}
	\usepackage{mathrsfs}
	\usepackage{enumerate}
	\usepackage{amsthm}
	\usepackage{float}\usepackage{nomencl}\usepackage{multirow}
	\usepackage{bm}
	\usepackage{lipsum}
	\usepackage{amssymb}
	\usepackage{algorithmic}
	\usepackage{epstopdf}
	\usepackage{caption}
	\usepackage{subcaption}\usepackage{xcolor}
	\usepackage{url}\usepackage{lettrine}
	\usepackage{balance}
	\usepackage{enumerate}
	\usepackage{amsmath}\usepackage{amssymb}
	\usepackage{amsthm}
	\usepackage{algorithm}
	\usepackage{graphicx}
	\usepackage{float}
	\usepackage{bm} 
	\usepackage{epstopdf}  
	
	\usepackage{doi}
	\usepackage{xcolor}
	\usepackage{mathtools}
	\usepackage{multirow}
	\usepackage{bm}
	\usepackage{makecell}

	\usepackage{mathrsfs} \usepackage{upgreek}
	\usepackage{mathrsfs}
	\usepackage{enumerate}
	\usepackage{amsmath}
	\usepackage{amsthm}
	\usepackage{caption}
	\usepackage{graphicx}\usepackage{lettrine}
	\usepackage{float}
	\usepackage{bm}\usepackage{array}
	\usepackage{algorithm}
	\usepackage{lipsum}
	\usepackage{subcaption}
	\usepackage{amssymb}
	\usepackage{epstopdf} 
	\usepackage{nomencl}\usepackage{stfloats}
	\usepackage{url}\usepackage{booktabs}\usepackage{colortbl}\usepackage{epstopdf}
	\usepackage{etoolbox}
	\def\BibTeX{{\rm B\kern-.05em{\sc i\kern-.025em b}\kern-.08em
			T\kern-.1667em\lower.7ex\hbox{E}\kern-.125emX}}
	
	\usepackage{nicefrac}       
	\usepackage{microtype}      
	\usepackage{lipsum}         
	\usepackage{doi}
	\usepackage{xcolor}
	\usepackage{mathtools}
	\usepackage{multirow}
	\usepackage{bm}
	\usepackage{makecell}

	\usepackage{array}
	\usepackage{textcomp}
	\usepackage{stfloats}
	\usepackage{url}
	\usepackage{verbatim}

	\newtheorem{assumption}{Assumption}
	\newtheorem{lemma}{Lemma}
	
	\newtheorem{theorem}{Theorem}

	\newtheorem{fact}{Fact}

	\def\bx{\bm{x}}
	
	\def\by{\bm{y}}
	
	\def\bw{\bm{w}}
	
	\def\bg{\bm{g}}
	\def\bG{\bm{G}}

	\def\bxi{\bm{\xi}}
	
	\def\E{\mathbb{E}}

	\def\tauit{\tau_i(t)}
	\def\etabar{\bar{\eta}}
	
	\title{Achieving Linear Speedup in Asynchronous Federated Learning with Heterogeneous Clients}
	\author{\IEEEauthorblockN{Xiaolu Wang},
		\IEEEauthorblockN{Zijian Li,~\IEEEmembership{Graduate Student Member, IEEE}},
		\IEEEauthorblockN{Shi Jin,~\IEEEmembership{Fellow, IEEE}},
		and \IEEEauthorblockN{Jun Zhang,~\IEEEmembership{Fellow, IEEE}}
				\thanks{X. Wang, Z. Li, and J. Zhang are with the Department of Electronic and Computer Engineering, The Hong Kong University of Science and Technology, Hong Kong SAR, China (E-mail: eexlwang@ust.hk, zijian.li@connect.ust.hk, eejzhang@ust.hk). 
					S. Jin is with the National Mobile Communications Research Laboratory, Southeast University, Nanjing, China (E-mail: jinshi@seu.edu.cn).
				(Corresponding author: Jun Zhang.)}
	}
	
\begin{document}
		
		\IEEEtitleabstractindextext{%
			\begin{abstract}
				Federated learning (FL) is an emerging distributed training paradigm that aims to learn a common global model without exchanging or transferring the data that are stored locally at different clients.
				The Federated Averaging (FedAvg)-based algorithms have gained substantial popularity in FL to reduce the communication overhead, where each client conducts multiple localized iterations before communicating with a central server. 
				In this paper, we focus on FL where the clients have diverse computation and/or communication capabilities.
				Under this circumstance, FedAvg can be less efficient since it requires all clients that participate in the global aggregation in a round to initiate iterations from the \textit{latest} global model, and thus the synchronization among fast clients and \textit{straggler clients} can severely slow down the overall training process.
				To address this issue, we propose an efficient asynchronous federated learning (AFL) framework called \textit{Delayed Federated Averaging (DeFedAvg)}. In DeFedAvg, the clients are allowed to perform local training with different stale global models at their own paces.
				Theoretical analyses demonstrate that DeFedAvg achieves asymptotic convergence rates that are on par with the results of FedAvg for solving nonconvex problems. More importantly, DeFedAvg is the first AFL algorithm that provably achieves the desirable \textit{linear speedup} property, which indicates its high scalability.
				Additionally, we carry out extensive numerical experiments using real datasets to validate the efficiency and scalability of our approach when training deep neural networks.
			\end{abstract}
			
		}
		\maketitle


		%
		%
		%
		%
		\IEEEraisesectionheading{\section{Introduction}\label{sec:introduction}}
		\IEEEPARstart{F}{ederated} learning (FL) is a distributed machine learning paradigm that enables training models across a network of clients (such as wireless sensors, mobile terminals, and edge devices) while keeping the data localized, under the orchestration of a central server \cite{kairouz2021advances}. 
		FL is particularly useful in mobile edge computing and Internet of Things (IoT) environments where the devices generate data locally. For example, FL enables real-time analytics and decision-making at the edge, empowering intelligent traffic management, environmental monitoring, and resource optimization in the context of smart cities \cite{pandya2023federated}. Besides, in industrial IoT, FL facilitates collaborative learning among sensors and devices, enabling anomaly detection, task scheduling, etc. \cite{abreha2022federated}. 
		
		In classical algorithms for distributed machine learning, e.g., parallel stochastic gradient descent (SGD), multiple training examples are processed simultaneously across different clients \cite{bertsekas2015parallel}. However, it may induce substantial communication overhead, as every local gradient needs to be sent to the server.
		Considering that federated networks may encompass an extensive array of clients, the communication within such networks can be significantly slower than local computation by several orders of magnitude, due to constrained resources like bandwidth, energy, and power \cite{li2020federated}.
		
		To address the pivotal communication bottleneck in FL, \textit{FedAvg} \cite{mcmahan2017communication}, which is also called \textit{local SGD} \cite{stich2019local}, performs multiple iterations locally at the clients before communicating with the server. 
		Compared with parallel SGD, the frequency of communication in FedAvg can be reduced by allowing \textit{local training} in each client. 
		An important feature of FedAvg is that it is able to achieve \textit{linear speedup} with respect to (wrt) the number of participating clients in each communication round \cite{yang2021achieving}, meaning that the communication overhead required to achieve a certain level of optimization accuracy decreases linearly with the increase in computational resources.
		Linear speedup is a highly desirable property in distributed learning, since it signifies the scalability of the optimization process as more clients are added.
		
		Typically, FL is faced with the challenge of \textit{system heterogeneity}, which refers to the variation in hardware, software, and network conditions across the participating clients \cite{zhong2022multi}. 
		In spite of the favorable practical and theoretical properties of FedAvg, a major challenge concerns the time needed to synchronize a communication round.
		In the presence of system heterogeneity, there can be large variations in the computation and communication speeds of different clients \cite{assran2020advances,li2020federated}. As a consequence, the time needed for a round is conditioned to the slowest client, while the faster clients stay idle once they send their local updates to the server. 
		This phenomenon, known as the \textit{straggler effect}, can significantly slow down the overall training process. 
		Motivated by this, a multitude of studies have focused on developing asynchronous federated learning (AFL) algorithms, where the clients are allowed to initiate local training at different times, thereby enhancing the system performance.
		Nevertheless, the introduction of asynchrony adds complexity to the learning dynamics, and existing AFL algorithms have not rigorously attained the desirable linear speedup property. This limitation can potentially restrict the scalability of these AFL algorithms.
		Therefore, it is imperative to address the following question: 
		{\it
			Is it possible to develop an AFL method that effectively mitigates the straggler effect while maintaining the linear speedup property observed in synchronous FL?
		}
		
		\subsection{Our Contributions}
		In this paper, we provide an affirmative answer to the aforementioned question. 
		Specifically, we propose an AFL framework,  called \textit{Delayed Federated Averaging (DeFedAvg)}, for efficient FL under system heterogeneity.
		To fit the FL scenarios with non-independent and identically distributed (non-IID) and independent and identically distributed (IID) data distributions across clients, we consider two types of DeFedAvg, i.e., DeFedAvg-nIID and DeFedAvg-IID. 
		To outline, DeFedAvg possesses the following main characteristics: 
		\begin{itemize}[leftmargin=*]
			\item  \textit{Partial Participation.} Only a subset of clients participate in the global iteration in each round.
			In DeFedAvg-nIID, the server sample $n$ clients uniformly random with replacement; while in DeFedAvg-IID, the server collects and aggregates these local model updates whenever $n$ of them become available.
			\item \textit{Delayed Local Training.} The clients keep receiving the latest global models  from the server and store them into their \textit{receive buffers} with overwriting permitted. They then perform local training, regardless of their participation in the global iteration. 
			The server accepts delayed updates without waiting for the participating clients to train using the latest global model.
			\item \textit{Asynchronous Clients.} Once a client completes its local training, the delayed update is either stored in the \textit{send buffer} until it is selected by the server (in DeFedAvg-nIID) or directly uploaded to the server (in DeFedAvg-IID). 
			The client then initiates new local training using the global model stored in its receive buffer, independent of the computation/communication states of other clients. 
		\end{itemize} 
		In a nutshell, DeFedAvg enhances the training efficiency compared to traditional FedAvg by reducing the server's waiting time and the clients' idle time.
		
		On the theoretical side, we provide the convergence analyses of DeFedAvg for solving smooth and nonconvex objective functions, including deep neural networks that are of particular interest.
		Our analyses reveal that DeFedAvg-nIID (respectively, DeFedAvg-IID) converges at a rate of  $\mathcal{O}({1}/{\sqrt{nT}} + {1}/{KT})$ (respectively, $\mathcal{O}({1}/{\sqrt{nKT}} + {1}/{KT})$) for sufficiently large $T$, 
		where $n$ is the number of participating clients in each round, $K$ is the number of local updates in each round, and $T$ is the total number of communication rounds. These imply the \textit{linear speedup} property wrt $n$.
		Our results demonstrate that incorporating delayed updates in local training effectively mitigates the straggler effect without sacrificing the asymptotic convergence rates.
		To our knowledge, we are the first to achieve the linear speedup guarantees in AFL.
		Furthermore, our theoretical findings for DeFedAvg can be reduced to the latest results for classic FedAvg and asynchronous SGD under different specific setups, adding to the versatility and applicability of our technical contributions.
		
		On the empirical side, we validate the effectiveness of DeFedAvg through extensive experiments involving the training of neural networks on real-world datasets. We compare its performance with that of the latest distributed learning algorithms, serving as a benchmark for evaluation. These experiments provide empirical evidence for the benefits and advantages of DeFedAvg in practical scenarios.

		\subsection{Related Works}
		\subsubsection{Linear Speedup Analyses of FedAvg Algorithms}
		Numerous variants of FedAvg have emerged in the literature since the groundbreaking study by McMahan et al. \cite{mcmahan2017communication}. In these algorithms, the main feature lies in the local training strategies, which aim at mitigating the communication costs involved in FL. 
		In the non-IID setting, the convergence rates of FedAvg have been established in \cite{haddadpour2019local} for problems satisfying the Polyak-{\L}ojasiewicz condition, and in \cite{wang2021cooperative,yang2021achieving} for general nonconvex problems.
		In the IID setting, the convergence of FedAvg algorithms has been studied in \cite{stich2019local,stich2020error,khaled2020tighter} for (strongly) convex problems, and in \cite{zhou2018convergence,yu2019parallel,stich2020error} for nonconvex problems. 
		More complicated variants of FedAvg include \cite{yu2019linear}, which employs momentum techniques to accelerate convergence, and \cite{karimireddy2020scaffold}, which incorporates control variates for variance reduction.
		The recent works by \cite{gu2021fast,wang2023linear} adopts a strategy that aggregates (possibly stale) stochastic gradients from \textit{all} clients in each round. 
		However, these approaches requires the server to wait for active clients to complete local computation using the \textit{current} global model.  
		While these algorithms achieve comparable convergence rates with linear speedup wrt the number of participating clients, they are all synchronous and rely on (a subset of) clients uploading up-to-date intermediate training results in each communication round. This makes them vulnerable to stragglers and essentially different from our DeFedAvg.
		
		\subsubsection{Asynchronous SGD Algorithms}
		There are various asynchronous variants of parallel SGD where clients operate at their own speeds without the need for synchronization, resulting in the server aggregating stale stochastic gradients. Under the IID setting, the linear speedup convergence of asynchronous SGD has been established in \cite{agarwal2011distributed} for convex objectives, in \cite{lian2015asynchronous,dutta2018slow} for nonconvex objectives, and in \cite{arjevani2020tight} for convex quadratic objectives.
		In the non-IID setting, \cite{gao2021provable} achieves linear speedup by assuming identical staleness of the aggregated gradients. On the other hand, \cite{koloskova2022sharper} provides the convergence analysis of asynchronous SGD for nonconvex objectives, while the linear speedup is not attained. 
		It is worth noting that, unlike DeFedAvg (as well as FedAvg), asynchronous SGD methods involve clients solely responsible for gradient computation without performing local training. As a result, asynchronous SGD methods may require more frequent communication with the server compared to DeFedAvg.
		
		\subsubsection{AFL Algorithms}\label{sec:AFL}
		AFL algorithms share similarities with asynchronous SGD, while AFL generally involves local training on the clients.
		The recent work called FedBuff \cite{nguyen2022federated} has demonstrated the state-of-the-art empirical performance in AFL under non-IID data. It adopts a similar strategy to our DeFedAvg, where the server's iteration occurs only when a specific subset of clients' delayed updates arrive. However, unlike the uniform client sampling in DeFedAvg-nIID, FedBuff caches local updates from \textit{any} $n$ active clients in each round.
		Based on the FedBuff framework, another AFL algorithm that caches all clients' historical updates was proposed in \cite{wang2023tackling}. While the convergence rate of this algorithm seems to achieve linear speedup, its theoretical analysis is intrinsically based on highly demanding assumptions. We will delve into a further discussion on this matter in Section \ref{sec:theory}.
		In several other works on AFL, such as \cite{xie2019asynchronous, chen2020asynchronous, fraboni2023general}, the server immediately updates the global model upon receiving a local model from a client; while in the most recently proposed AFL algorithms, called QuAFL \cite{zakerinia2022quafl} and FAVANO \cite{leconte2023favas}, a different asynchrony model is adopted, where the server can interrupt the clients to request their intermediate local updates.
		Besides, there is another line of research on AFL that focuses on the semi-decentralized edge systems \cite{ma2021fedsa,sun2023semi}.
		Although the existing papers on AFL generally provide convergence analyses on nonconvex problems, the linear speedup property has still not been rigorously attained under standard assumptions.
		
		\section{Delayed Federated Averaging}\label{sec:algo}
		The goal of FL is typically cast as minimizing the following global objective function:
		\begin{align}
			\min_{\bw \in \mathbb{R}^d} F(\bw) \coloneqq \frac{1}{N} \sum_{i=1}^{N} F_i(\bw),
			\label{eq:prob}
		\end{align}
		where $F_i(\bw) \coloneqq \E_{\bxi_i \sim \mathcal{D}_i} [f(\bw;\bxi_i)]$ is the local objective function and $N$ is the number of clients.
		Here, ${\cal D}_i$ represents the data distribution accessible by client $i$ supported on sample space $\Xi_i$ and $f( \cdot ; \bxi_i ): \mathbb{R}^d \rightarrow \mathbb{R}$ is client $i$'s loss function that is continuously differentiable for any given data sample $\bxi_i \in \Xi_i$.
		In this paper, we focus on the situations in which the $N$ clients are heterogeneous and can communicate with the central server without peer-to-peer communication.
		
		FedAvg has garnered great interest in the field of FL lately, and there exist a batch of distributed algorithms  in the literature that can be seen as variants of FedAvg, e.g., \cite{mangasarian1995parallel,zinkevich2010parallelized,stich2019local,stich2020error,karimireddy2020scaffold,yang2021achieving,li2020federatedoptimization,wang2020tackling}. 
		The general principle of FedAvg algorithms lies in performing iterations for multiple steps with the local data at the clients, and then aggregating the model updates at the central server.
		Here, let us recap a generic version of FedAvg that was introduced in \cite[Section 3]{karimireddy2020scaffold}. 
		Specifically, in communication round $t$ ($t \geq 0$), the server samples a client set $\mathcal{I}_t$ with $\mathcal{I}_t = n$ uniformly with replacement out of the $N$ clients, and then the $n$ participating clients start with $\bw_i^{t,0} = \bw^t$ and parallelly execute $K$-step SGD during round $t$ according to the following update rule:
		\begin{align}
			\bw_i^{t,k+1} &= \bw_i^{t,k} - \etabar \nabla f ( \bw_i^{t,k}; \bxi_i^{t,k} )
			\label{eq:localsgd1}
		\end{align}
		for $k = 0,1,\dots,K-1$. In the midst, $\bxi_i^{t,k}$ are independently sampled according to distribution $\mathcal{D}_i$ and $\etabar>0$ is the local learning rate.
		Once the $n$ clients finish their local training and upload their last iterates $\{\bw_i^{t,K}: i \in \mathcal{I}_t\}$, the server proceeds with the global iteration using the following update rule:
		\begin{align}
			\bw^{t+1} = \bw^t - \frac{\eta}{n} \sum_{i \in \mathcal{I}_t} (\bw_i^{t,K} - \bw^t),
			\label{eq:localsgd2}
		\end{align}
		where $\eta>0$ is the global learning rate. 
		Note that the FedAvg procedures \eqref{eq:localsgd1}--\eqref{eq:localsgd2} with $\eta=1$ reduce to the vanilla FedAvg algorithm proposed by the seminal work \cite{mcmahan2017communication}.
		
		Equipped with the description of this (synchronous) FedAvg, we present the algorithmic details of DeFedAvg for solving Problem \eqref{eq:prob} under both the non-IID ($\mathcal{D}_i \neq \mathcal{D}_j$ for $i \neq j$) and IID ($\mathcal{D}_i = \mathcal{D}_j$ for all $i,j \in [N]$) settings.
		
		\subsection{DeFedAvg-nIID}\label{sec:alg-hetero}
		We first introduce DeFedAvg-nIID for solving Problem \eqref{eq:prob} with non-IID local data distributions $\mathcal{D}_1,\dots,\mathcal{D}_N$, which is common in most FL applications. 
		We describe the pseudo-code of DeFedAvg-nIID in Algorithm \ref{algo2}, in which the clients and the server work in a parallel manner. 
		
		In the initial round $t=0$ of DeFedAvg-nIID, the server initializes a global model $\bw^0$ and broadcasts it to the \textit{receive buffers} of all the clients. Each client fetches the global model $\bw^0$ from their receive buffer and perform parallel local training starting from $\bw^0$.
		Within a general round $t$ ($t \geq 0$), we elaborate Algorithm \ref{algo1} in the following.
		
		\subsubsection{Clients' Procedures}
		Once the receive buffer of a client is non-empty, it can retrieve the global model $\bw$ (ignoring the superscript for now) from the receive buffer and commence local training using its local data. 
		At the same time, the receive buffer may be updated with the server's latest global model.
		The completed local model updates (which will be specified later) are then placed into the \textit{send buffer}. As a result, each client perform asynchronous local training at its own pace without the needs to be selected by the server and synchronized with other clients.
		
		Suppose that client $i$ is selected by the server at the beginning of round $t$. The latest global model at the receive buffer is $\bw^t$, while it is possible that the client's local training has been initiated from a stale version of the global model, denoted by $\bw^{\tauit}$. 
		Here, $\tauit$ is the round index of the global model that client $i$ has used in local training in round $t$, and thus we have $0 \leq \tauit \leq t$.
		Specifically, the initial local iterate is set as $\bw_i^{\tauit,0} \coloneqq \bw^{\tauit}$, then the local sequence $\bw_i^{\tauit,1},\dots,\bw_i^{\tauit,K}$ is generated through the following $K$-step SGD iterations ($K \geq 1$):
		\begin{align}
			\bw_i^{\tauit,k+1} &= \bw_i^{\tauit,k} - \etabar \nabla f ( \bw_i^{\tauit,k}; \bxi_i^{t,k} )
			\label{eq:local-iter'}
		\end{align}
		for $k = 0,1,\dots,K-1$, where $\etabar>0$ is the local learning rate and $\bxi_i^{t,k}$'s are independently sampled from $\mathcal{D}_i$.
		Then, client $i$ place its local model update, defined as
		\begin{align}
			\bm{\Delta}_i^t \coloneqq \bw_i^{\tauit,0} - \bw_i^{\tauit,K},
			\label{eq:delta}
		\end{align}
		into the send buffer before uploading to the server.
		
		\begin{algorithm}[t]
			\caption{DeFedAvg-nIID}
			\label{algo2} 
			\vspace{1mm}
			{\sf Client $i$ executes:}
			\begin{algorithmic}[1]
				\STATE \textbf{Input:} $K \in \mathbb{Z}_+$, $\etabar > 0$
				\IF{the receive buffer is nonempty}
				\STATE Retrieve $\bw$ from the receive buffer and set $\bw_i \leftarrow \bw$
				\FOR{$k=0,1,\dots,K-1$}
				\STATE Sample $\bxi_i \sim \mathcal{D}_i$
				\STATE Set $\bw_i \leftarrow \bw_i - \etabar \nabla f (\bw_i; \bxi_i)$.
				\ENDFOR
				\STATE Set $\bm{\Delta}_i = \bw - \bw_i$ and update the send buffer with $\bm{\Delta}_i$
				\ENDIF
				\IF{client $i$ is selected by the server}
				\WHILE{the send buffer is nonempty}
				\STATE Send $\bm{\Delta}_i$ to the server 
				\ENDWHILE
				\ENDIF
			\end{algorithmic}
			{\sf Server executes:}
			\begin{algorithmic}[1] 
				\STATE \textbf{Input:} $\bw^0 \in \mathbb{R}^d$, $n,T \in \mathbb{Z}_+$, $\eta > 0$
				\STATE Broadcast $\bw^0$ all the clients
				\FOR{$t=0,1,\dots,T-1$}
				\STATE Sample $n$ clients out of $[N]$ uniformly with replacement, and form a set $\mathcal{I}_t$
				\STATE Collect the local updates $\bm{\Delta}_i$ from all clients $i \in \mathcal{I}_t$ 
				\STATE Set $\bm{\Delta}_i^t \leftarrow \bm{\Delta}_i$ for $i \in \mathcal{I}_t$
				\STATE Aggregation: $\bm{\Delta}^t \leftarrow \frac{1}{n} \sum_{i \in \mathcal{I}_t} \bm{\Delta}_i^t$
				\STATE Global iteration: $\bw^{t+1} \leftarrow \bw^t - \eta \bm{\Delta}^t$
				\STATE Broadcast $\bw^{t+1}$ to all the $N$ clients
				\ENDFOR
			\end{algorithmic}
		\end{algorithm}
		
		\subsubsection{Server's Procedures}
		When the clients have diverse communication/computation speeds, the deterministic client participation scheme employed in many FL algorithms \cite{li2019convergence,nguyen2022federated} tends to favor fast clients, as the fast clients may dominate the global aggregation compared to the slow ones throughout the training process. This can result in a biased exploitation of the local information in $F_i$'s, potentially impeding the convergence of these algorithms.
		Therefore, we adopt a random client selection scheme in DeFedAvg-nIID, where the server samples clients uniformly with replacement $n$ independent times in each round.
		The selected clients upload their local updates if their send buffers are nonempty, while the remaining clients continue their ongoing work. 
		We use $\mathcal{I}_t$ to denote the random set of the selected clients in round $t$. Since each client can be selected more than once, $\mathcal{I}_t$ is a multiset that may contain repeated elements and $|\mathcal{I}_t| \leq n$.
		Once the clients in $\mathcal{I}_t$ have uploaded their local updates $\bm{\Delta}_i^t$'s, the server proceeds to create a new global model $\bw^{t+1}$ as follows:
		\begin{align}
			\bw^{t+1} = \bw^t - \eta \bm{g}_{\textnormal{nIID}}^t,
			\label{eq:global-iter'}
		\end{align}
		where $\eta>0$ is the global learning rate, $\bm{\Delta}_i^t$ is client $i$'s local update that will be specified later, and
		\begin{align}
			\bm{g}_{\textnormal{nIID}}^t \coloneqq \frac{1}{|\mathcal{I}_t|} \sum_{i \in \mathcal{I}_t} \bm{\Delta}_i^t.
			\label{eq:g-niid}
		\end{align}
		Then, $\bw^{t+1}$ is broadcast to all the $N$ clients' \textit{receive buffers}. 
		As a result of formulas \eqref{eq:local-iter'}--\eqref{eq:g-niid}, the global iterations of the DeFedAvg-nIID algorithm can be expressed compactly as:
		\begin{align}
			\bw^{t+1} = \bw^t - \frac{\eta \etabar}{n} \sum_{i \in \mathcal{I}_t} \sum_{k=0}^{K-1} \nabla f ( \bw_i^{\tauit,k}; \bxi_i^{t,k} )
			\label{eq:compact}
		\end{align}
		for $t = 0,1,\dots, T-1$. In this iteration formula, the staleness in the aggregated gradients reflects the asynchrony property, and the uniform randomness in $\mathcal{I}_t$ plays a crucial role in ensuring linear speedup with non-IID data.
		
		\begin{algorithm}[t]
			\caption{DeFedAvg-IID}
			\label{algo1}
			\vspace{1mm} 
			{\sf Client $i$ executes:}
			\begin{algorithmic}[1] 
				\STATE \textbf{Input:} $K \in \mathbb{Z}_+$, $\etabar > 0$
				\IF{the receive buffer is nonempty}
				\STATE Retrieve $\bw$ from the receive buffer and set $\bw_i \leftarrow \bw$
				\FOR{$k=0,1,\dots,K-1$}
				\STATE Sample $\bxi_i \sim \mathcal{D}$
				\STATE  Set $\bw_i \leftarrow \bw_i - \etabar \nabla f (\bw_i; \bxi_i)$.
				\ENDFOR
				\STATE Send $\bm{\Delta}_i = \bw - \bw_i$ to the server
				\ENDIF
			\end{algorithmic}
			{\sf Server executes:}
			\begin{algorithmic}[1]
				\STATE \textbf{Input:} $\bw^0 \in \mathbb{R}^d$, $n,T \in \mathbb{Z}_+$, $\eta > 0$
				\STATE Broadcast $\bw$ to all the clients
				\FOR{$t=0,1,\dots,T-1$} 
				\STATE Collect $n$ local model updates $\bm{\Delta}_i$'s from a set $I_t$ of clients that first complete local training
				\STATE Set $\bm{\Delta}_i^t \leftarrow \bm{\Delta}_i$ for $i \in I_t$
				\STATE Aggregation: $\bm{\Delta}^t \leftarrow \frac{1}{n} \sum_{i \in {I}_t} \bm{\Delta}_i^t$
				\STATE Global iteration: $\bw^{t+1} \leftarrow \bw^t - \eta \bm{\Delta}^t$
				\STATE Broadcast $\bw^{t+1}$ to all the $N$ clients
				\ENDFOR
			\end{algorithmic}
		\end{algorithm}
		
		\subsection{DeFedAvg-IID}\label{sec:alg-iid}
		In this subsection, we consider the IID setting where each client has a local dataset that is representative of the overall population and thus the local loss functions are the same ($F_i = F_j$ for all $i,j \in [N]$).
		In edge computing environments, clients such as sensors or cameras may be deployed in a controlled environment or have similar characteristics. For instance, in a smart city application where traffic monitoring cameras contribute to FL, the data collected from these cameras might be IID if they are installed in similar locations and conditions. 
		Besides, in certain IoT applications, such as smart home clients or wearable devices produced by the same manufacturer, the data collected from these clients may exhibit similar patterns and characteristics. If the clients are part of a homogeneous group, the data distributions may be relatively consistent, assuming similar usage patterns.
		
		While DeFedAvg-nIID can be applied to such a scenario, its uniform client sampling scheme is conservative because it may overlook faster clients that complete local training but are not selected by the server.
		Therefore, we develop DeFedAvg-IID to further enhance the system efficiency.
		Formally, we assume that the local data distribution $\mathcal{D}_i \coloneqq \mathcal{D}$ for all $i\in[N]$ with $\mathcal{D}$ supported on some sample space $\Xi$. Then, Problem \eqref{eq:prob} can be simplified as follows:
		\begin{align}
			\min_{\bw \in \mathbb{R}^d} F(\bw) = \E_{\bxi\sim\mathcal{D}} [f(\bw;\bxi)].
			\label{eq:prob-iid}
		\end{align}
		The pseudo-code of DeFedAvg-IID is described in Algorithm \ref{algo1}.
		The warm-up phase of DeFedAvg-IID remains the same as described in Section \ref{sec:alg-hetero}, and we focus on the key distinctions in terms of the clients' and the server's procedures. 
		
		\subsubsection{Clients' Procedures}
		The clients' local iteration formula of DeFedAvg-IID is the same as that of DeFedAvg-nIID given by \eqref{eq:local-iter'}, while the different thing is that the clients do not need to have send buffers. Once client $i$'s local training is completed, the local model update $\bm{\Delta}_i^t$ shall be directly sent to the server without being selected.
		Subsequently, if the receive buffer is nonempty, the client proceeds with new local training by fetching the latest global model from the receive buffer. 
		
		\subsubsection{Server's Procedures} 
		In DeFedAvg-IID, the server does not actively sample clients randomly, while cache \textit{arbitrary} $n$ clients' local model updates that first arrive. 
		On the one hand, the dominance of fast clients during the training process is not an issue, as the data samples from fast and slow clients are equally representative of the overall population. 
		On the other hand, such an arbitrary client participation scheme fully exploits the computation/communication capability of fast clients, which potentially reduces the server's waiting time in DeFedAvg-nIID.
		Formally speaking, we let $I_t$ be the set of participating clients in round $t$ with $|I_t| = n$. We remark that $I_t$ does not contain randomness, in contrast to the random set $\mathcal{I}_t$ (denoted in calligraphic) that was defined earlier. Similar to \eqref{eq:global-iter'}, the server creates a new global model $\bw^{t+1}$ as follows:
		\begin{align}
			\bw^{t+1} = \bw^t - \eta \bm{g}_{\textnormal{IID}}^t,
			\nonumber
		\end{align}
		where $\eta>0$ is the global learning rate, $\bm{\Delta}_i^t$ is client $i$'s local model update that is same as \eqref{eq:delta}, and 
		$\bm{g}_{\textnormal{IID}}^t \coloneqq \frac{1}{|I_t|} \sum_{i \in I_t} \bm{\Delta}_i^t$.
		Then, $\bw^{t+1}$ is broadcast to all the $N$ clients' \textit{receive buffers}. 
		The global iterations of DeFedAvg-IID can be compactly expressed as:
		\begin{align*}
			\bw^{t+1} = \bw^t - \frac{\eta \etabar}{n} \sum_{i \in I_t} \sum_{k=0}^{K-1} f ( \bw_i^{\tauit,k}; \bxi_i^{t,k} )
		\end{align*}
		for $t = 0,1,\dots, T-1$, which is same as \eqref{eq:compact} expect that the stale gradients are aggregated over non-random sets $I_t$.
		
		\begin{figure}[t]
			\centering
			\begin{subfigure}{0.48\textwidth}
				\centering
				\includegraphics[width=\linewidth]{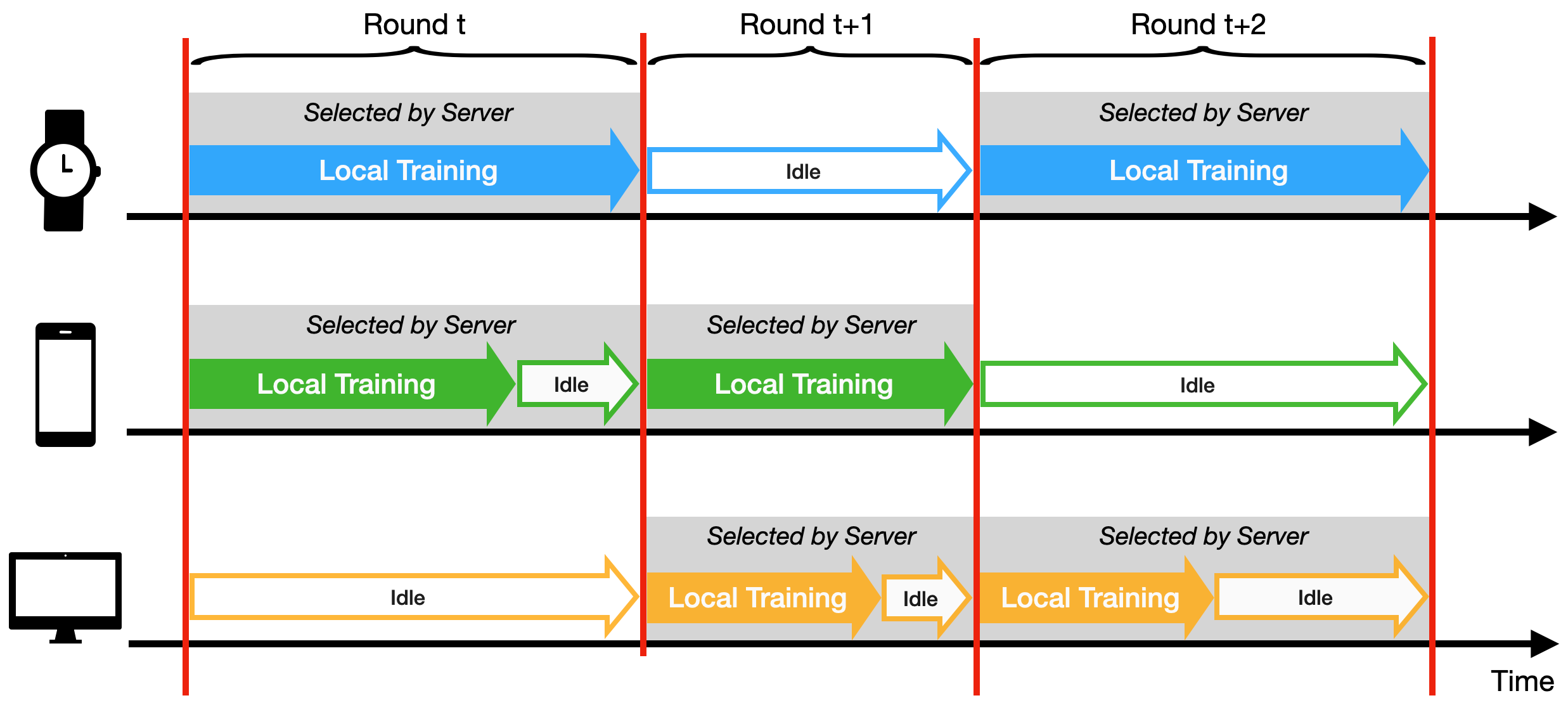}
				\caption{Clients' local training in FedAvg.}
				\vspace{3mm}
				\label{fig:fedavg}
			\end{subfigure}
			
			\begin{subfigure}{0.48\textwidth}
				\centering
				\includegraphics[width=\linewidth]{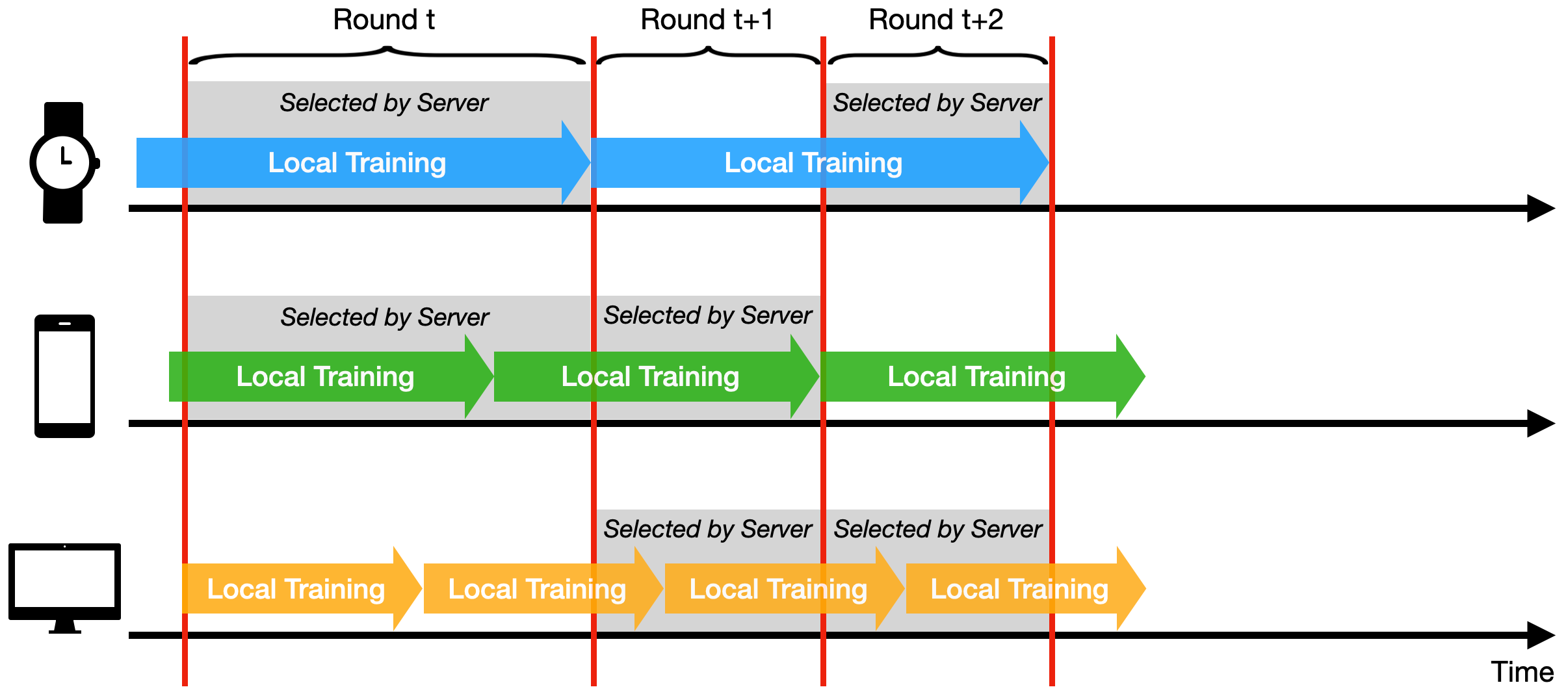}
				\caption{Clients' local training in DeFedAvg-nIID.}
				\vspace{3mm}
			\end{subfigure}
			
			\begin{subfigure}{0.48\textwidth}
				\centering
				\includegraphics[width=\linewidth]{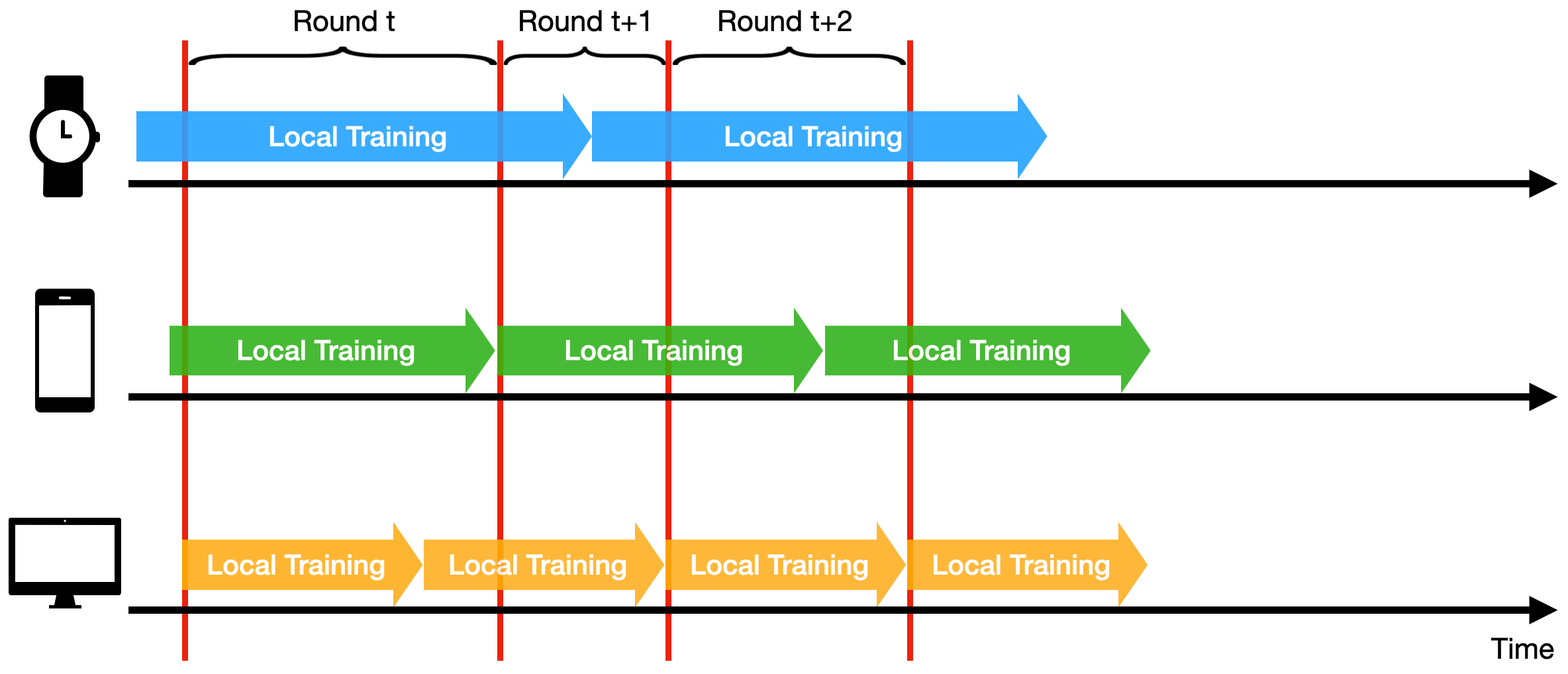}
				\caption{Clients' local training in DeFedAvg-IID.}
				\vspace{3mm}
			\end{subfigure}
			\caption{Comparison of clients' local training protocols of (synchronous) FedAvg, DeFedAvg-nIID, and DeFedAvg-IID.}
			\label{fig:protocols}
		\end{figure}
		
		\subsection{Practical Issues of DeFedAvg.}
		Alike its synchronous FL counterparts \cite{mcmahan2017communication,yang2021achieving}, DeFedAvg also adopts partial clinet participation schemes, which subsume the ``full'' client participation scheme as a special case by setting $n = N$.
		In spite of the partial participation, DeFedAvg requires the server to broadcast the global model to all clients in each round. This is because all the clients need to use it for local training, regardless of whether they are selected or not. In many FL applications, the central server (e.g., a base station) has sufficient power supply and downlink bandwidth for broadcasting to all the clients (e.g., smartphones, autonomous vehicles). The clients typically have limited power supply and uplink bandwidth, while fortunately only unicast communication to the server is required in DeFedAvg.
		
		When implementing DeFedAvg-nIID, each client is able to continuously execute local training using the the most recent global model in the receive buffer, so that the potential delay $t-\tauit$ can be minimal when the client is selected in a certain round $t$. 
		Using this strategy, a client may overwrites its send buffer if it generates multiple local model updates before being selected by the server. However, under the circumstances where the clients are energy-limited, we can let the send buffer be overwritten only when the client is sampled by the server. Indeed, both implementations enjoy the same theoretical property (see Theorems \ref{thm:hetero} and \ref{thm:iid} in Section \ref{sec:theory}), while may differ in practical efficiency.
		
		We present an intuitive comparison of clients' local training protocols of FedAvg, DeFedAvg-nIID, and DeFedAvg-IID, as shown in Fig. \ref{fig:protocols}. We assume that there are a total of $N=3$ clients, and in each round, only $n=2$ clients participate in the global update, and we ignore the communication time for the sake of simplicity.	
		We can observe that the DeFedAvg algorithms effectively mitigate the straggler effect that is present in FedAvg by reducing the idle time during the training process. 
		Furthermore, DeFedAvg-IID goes a step further by eliminating the cost associated with random client selection that exists in DeFedAvg-nIID. By using the `first-complete-first-upload' pattern for client participation, DeFedAvg-IID achieves even faster training compared to DeFedAvg-nIID, while it is only suitable for the IID setting.
		
		\section{Theoretical Analyses}\label{sec:theory}
		In this section, we establish the convergence rates and linear speedup of DeFedAvg-nIID and DeFedAvg-IID. We first introduce the following standard assumptions before presenting our theoretical results.
		
		\begin{assumption}[$L$-smoothness]\label{as:smooth}
			There exists $L>0$ such that for all $i \in [N]$ and $\bw,\bw'\in\mathbb{R}^d$,
			\begin{align}
				\| \nabla F_i(\bw) - \nabla F_i(\bw') \|_2 \leq L \| \bw - \bw' \|_2.
				\nonumber
			\end{align} 
		\end{assumption}
		
		\begin{assumption}[unbiased stochastic gradient]\label{as:unbias}
			Suppose that $\bxi_i \in \Xi_i$ is a data sample, $\bw \in \mathbb{R}^d$ is a local/global iterate generated by DeFedAvg, and $\mathcal{F}$ is a sigma algebra. 
			If $\bw$ is $\mathcal{F}$-measurable and $\bxi_i$ is independent of $\bw$, then for all $i \in [N]$,
			\begin{align}
				\E [\nabla f (\bw; \bxi_i) \mid \mathcal{F}] = \nabla F_i (\bw).
				\label{eq:unbiased}
			\end{align}
		\end{assumption}
		
		\textit{Subtlety in Applying Assumption \ref{as:unbias}:}
		Note that if $\bw$ is random, then $\E_{\bxi_i \sim \mathcal{D}_i} [\nabla f (\bw; \bxi_i)] = \E [\nabla f (\bw; \bxi_i) | \bw] \neq \nabla F_i (\bw)$ as the iterate $\bw$ can be a function of (thus depend on) $\bxi_i$ in general.
		To ensure the unbiasedness property \eqref{eq:unbiased}, we need the conditions that $\mathcal{F}$ contains all information in $\bw$ and $\bxi_i$ is independent of $\bw$. 
		Due to the presence of delay in DeFedAvg, it is crucial to deal with the total expectation of the inner product term $\E \left[\left\langle \nabla F(\bw^t), \nabla f ( \bw_i^{\tauit,k}; \bxi_i^{t,k} ) \right\rangle\right]$ in the theoretical analysis.
		It is worth emphasizing that we index the data sample $\bxi_i^{t,k}$ by $t$ to highlight that all previous models $\bw^0,\bw^1,\dots,\bw^t$ do not depend on it. 
		Thus, applying the law of total expectation and Assumption \ref{as:unbias} yields
		\begin{align*}
			&~\E \left[\left\langle \nabla F(\bw^t), \nabla f ( \bw_i^{\tauit,k}; \bxi_i^{t,k} ) \right\rangle\right]
			\nonumber
			\\
			=&~ \E \left[\left\langle \nabla F(\bw^t), \E \left[ \nabla f ( \bw_i^{\tauit,k}; \bxi_i^{t,k} ) ~\middle|~ \bw^t,\bw_i^{\tauit,k} \right] \right\rangle\right]
			\nonumber
			\\
			=&~ \E \left[\left\langle \nabla F(\bw^t), \nabla F_i ( \bw_i^{\tauit,k} ) \right\rangle\right],
		\end{align*}
		which plays an important role in establishing the convergence rates and linear speedup (see the supplementary materials for further details).
		By contrast, some existing AFL algorithms \cite{avdiukhin2021federated,zhang2023no,wang2023tackling} involve inner product terms of the form $\E \left[\left\langle \nabla F(\bw^t), \nabla f ( \bw_i^{\tauit,k}; \bxi_i^{s,k} ) \right\rangle\right]$, where $s < t$. 
		Since $\bxi_i^{s,k}$ was sampled in an earlier round, the current global model $\bw^t$ contains information about $\bxi_i^{s,k}$. Hence,
		\begin{align*}
			&~\E \left[\left\langle \nabla F(\bw^t), \nabla f ( \bw_i^{\tauit,k}; \bxi_i^{s,k} ) \right\rangle\right]
			\nonumber
			\\
			=&~ \E \left[\left\langle \nabla F(\bw^t), \E \left[ \nabla f ( \bw_i^{\tauit,k}; \bxi_i^{s,k} ) ~\middle|~ \bw^t,\bw_i^{\tauit,k} \right] \right\rangle\right]
			\nonumber
			\\
			\neq&~ \E \left[\left\langle \nabla F(\bw^t), \nabla F_i ( \bw_i^{\tauit,k} ) \right\rangle\right].
		\end{align*}
		This subtle yet fatal discrepancy might be overlooked by previous works, which poses a significant challenge to conduct rigorous convergence analyses of these algorithms under the standard unbiasedness Assumption \ref{as:unbias}.
		
		Furthermore, we impose the following upper bounds on the conditional variance of stochastic gradients and maximum delay. 
		\begin{assumption}[bounded variance]\label{as:bounded-var}
			Suppose that $\bxi_i \in \Xi_i$ is a data sample, $\bw \in \mathbb{R}^d$ is a local/global iterate generated by DeFedAvg, and $\mathcal{F}$ is a sigma algebra. 
			If $\bw$ is $\mathcal{F}$-measurable and $\bxi_i$ is independent of $\bw$, then there exists $\sigma > 0$ such that for all $i \in [N]$ and $k = 0, 1, \dots, K-1$,
			\begin{align*}
				\E \left[\| \nabla f (\bw_i; \bxi_i) - \nabla F_i (\bw_i)  \|_2^2 \mid \mathcal{F} \right] \leq \sigma^2.
			\end{align*} 
		\end{assumption}
		
		\begin{assumption}[bounded delay]\label{as:bounded-delay}
			There exists an integer ${\lambda} \geq 1$ such that for all $i \in [N]$ and $t \geq 0$, 
			\[
			\tauit \geq (t - {\lambda})_+,
			\]
			where $(x)_+ \coloneqq \max\{x, 0\}$ for $x \in \mathbb{R}$.
		\end{assumption}
		Assumption \ref{as:bounded-delay} is common in distributed learning using stale gradients, e.g., \cite{lian2015asynchronous,nguyen2022federated}, which indicates that the clients can be arbitrarily delayed as long as they participate in the global iterations at least once in the last ${\lambda}$ rounds.
		For Algorithm \ref{algo1}, the value of $\lambda$ is a system attribute that is determined by the computation/communication time of the clients and the server. 
		For Algorithm \ref{algo2}, due to the randomness in client sampling, there is a chance (albeit tiny) that a client is not selected throughout the training process and thus Assumption \ref{as:bounded-delay} is violated. 
		Indeed, according to \cite[Theorem 5.2]{gu2021fast}, it holds with probability at least $1-\delta$ that $\lambda \leq \mathcal{O} \left( \frac{1}{p} \left( 1+\log\frac{NT}{\delta} \right) \right)$, where $p = 1 - \left(\frac{N-1}{N}\right)^n$ is the probability of each client being sampled in a single round.
		In other words, the maximum delay ${\lambda}$ in DeFedAvg-nIID exhibits mild logarithmic dependence on the total number of communication rounds $T$ with high probability. 
		
		\subsection{Linear Speedup of DeFedAvg-nIID}\label{sec:thm-hetero}
		Since DeFedAvg-nIID applies to Problem \eqref{eq:prob} with different local objective functions, we impose the following upper bound on the heterogeneity of $F_i$'s:
		
		\begin{assumption}[bounded heterogeneity]\label{as:bounded-hetero}
			Suppose that $\bw \in \mathbb{R}^d$ is a local/global iterate generated by Algorithm \ref{algo2}. Then, there exists $\nu \geq 0 $ such that for all $i \in [N]$, 
			\begin{align}
				\E \| \nabla F_i(\bw) - \nabla F(\bw) \|_2^2 \leq \nu^2.
				\nonumber
			\end{align}
		\end{assumption}
		Moreover, we assume the uniform upper bound on the expected gradient norm squares:
		\begin{assumption}[bounded gradient]\label{as:bounded-grad}
			Suppose that $\bw \in \mathbb{R}^d$ is a local/global iterate generated by Algorithm \ref{algo2}. Then, there exists $G \geq 0 $ such that for all $i \in [N]$,
			\begin{align}
				\E \| \nabla F_i(\bw) \|_2^2 \leq G^2.
				\nonumber
			\end{align}
		\end{assumption}
		Now, we are ready to present the following convergence rate of DeFedAvg-nIID: 
		\begin{theorem}\label{thm:hetero}
			Let $F^* \in \mathbb{R}$ be the optimal function value of Problem \eqref{eq:prob}, $\bw^0 \in \mathbb{R}^d$ be the initial global model of Algorithm \ref{algo2}, and $\{\bw^t\}_{t \geq 1}$ be the sequence of global models generated by Algorithm \ref{algo2}. 
			Let 
			$
			\eta = \sqrt{4 n K (F(\bw^0)- F^*)}
			~\text{and}~
			\etabar = {1}/{\sqrt{(4 \sigma^2 L + 2 L K G^2)T}K}.
			$
			Suppose that Assumptions \ref{as:smooth}--\ref{as:bounded-grad} hold and  
			\begin{align}
				\etabar \hspace{-1mm}\leq\hspace{-1mm} \frac{1}{8 L K}
				\hspace{0.2mm}\text{and}\hspace{1mm}
				\eta \etabar \hspace{-1mm}\leq\hspace{-1mm}
				\min \hspace{-1mm}\left\{\hspace{-1mm} \frac{1}{4 L K \lambda},\hspace{-0.5mm} \frac{2 \sigma^2 + G^2 K}{8 \sigma^2 \hspace{-0.5mm} L K \lambda \hspace{-1mm}+\hspace{-1mm} 4 L G^2 K^2 {\lambda}^2} \hspace{-1mm}\right\}\hspace{-1mm}.
				\label{eq:lr'}
			\end{align}
			Then, it holds for $T \geq 1$ that
			\begin{align}
				\frac{1}{T} \hspace{-1mm}\sum_{t=0}^{T-1} \hspace{-0.5mm}\E \| \nabla \hspace{-0.5mm} F(\bw^t) \|_2^2\hspace{-0.5mm}
				\leq& \sqrt{\hspace{-0.5mm}\frac{256 (F(\bw^0)\hspace{-0.5mm}-\hspace{-0.5mm}F^*)(16 \sigma^2 L \hspace{-0.8mm}+\hspace{-0.8mm} 8 L G^2 K)}{n K T}}
				\nonumber
				\\
				&+ \frac{(\sigma^2+8K\nu^2) 8 L^2}{(4 \sigma^2 L + 2 L K G^2) K T}.
				\label{eq:thm2}
			\end{align} 
		\end{theorem}
		
		The proof of Theorem \ref{thm:hetero} is provided in the supplementary materials. 
		On the one hand, the global learning rate $\eta$ given by Theorem \ref{thm:hetero} is positively correlated with $F(\bw^0)-F^*$.
		On the other hand, the local learning rate $\bar{\eta}$ is negatively correlated with the gradient's variance bound $\sigma^2$, meaning that the noisier the local gradient is, the smaller the local learning rate should be to avoid ``client drift'' \cite{karimireddy2020scaffold}.
		
		As a consequence of condition \eqref{eq:lr'}, we need
		$T \geq \Omega(n K \lambda^4)$.
		Therefore, the bound \eqref{eq:thm2} indicates that the average squared gradient norm converges to zero at a rate of $\mathcal{O}({1}/{\sqrt{nT}} + 1/KT)$ for sufficiently large $T$. 
		In other words, to ensure that $\frac{1}{T} \sum_{t=0}^{T-1} \E | \nabla F(\bw^t) |_2^2 \leq \epsilon$, the total number of communication rounds required scales as $\mathcal{O}(1/n\epsilon^2 + 1/K\epsilon)$. This scaling is inversely proportional to the number of participating clients $n$ for sufficiently small $\epsilon$, indicating linear speedup with respect to $n$. In essence, as the number of participating clients increases, the required number of communication rounds decreases, leading to improved convergence rates and overall efficiency in FL.
		
		When we set $\lambda = 1$, indicating no delay in the updates, Theorem \ref{thm:hetero} aligns with the existing result for synchronous FedAvg for FL with non-IID data.
		Specifically, it has been shown in \cite{yang2021achieving} that the FedAvg converges at a rate of $\mathcal{O}({\sqrt{K}}/{\sqrt{nT}} + 1/T)$ when $T = \Omega(n K)$ for solving Problem \eqref{eq:prob}. 
		This observation highlights that the delay in local training does not hinder the asymptotic convergence rate when compared to synchronous FedAvg.
		
		Consider the special case when $K=1$, then DeFedAvg-nIID reduces to an asynchronous SGD algorithm, and Theorem \ref{thm:hetero} implies a $\mathcal{O}({1}/{\sqrt{nT}} + 1/T)$ convergence rate for $T=\Omega(n\lambda^4)$. 
		We notice that \cite{koloskova2022sharper} has established the best-known convergence rate of $\mathcal{O}(1/\sqrt{T} + 1/T)$ for asynchronous SGD under the same setting without the bounded gradient assumption, while the linear speedup property remains unclear.
		Our result, as far as we know, provides the first linear speedup guarantee of asynchronous SGD with non-IID data, which can be of independent interest.
		
		DeFedAvg-nIID is most analogous to the FedBuff AFL algorithm \cite{nguyen2022federated}, which also adopts asynchronous local training and partial participation schemes. 
		Differently, FedBuff caches local updates from any $n$ active clients in each round. Such a deterministic client participation scheme favors fast clients and thus may impede convergence under the heterogeneous setting.
		This issue is also evidenced by the theoretical analysis in \cite[Section D]{nguyen2022federated}, which has to make a rather demanding assumption that all clients participate with equal probability so as to guarantee the convergence rate of $\mathcal{O}(1/\sqrt{KT} +K/(n^2 T) + \lambda^2/T)$, without achieving asymptotic linear speedup wrt $n$. 
		
		\subsection{Linear Speedup of DeFedAvg-IID}\label{sec:thm-homo}
		Now, we present the following convergence rate of DeFedAvg-IID:
		\begin{theorem}\label{thm:iid}
			Let $F^* \in \mathbb{R}$ be the optimal function value of Problem \eqref{eq:prob-iid}, $\bw^0 \in \mathbb{R}^d$ be the initial global model of Algorithm \ref{algo1}, and $\{\bw^t\}_{t \geq 1}$ be the sequence of global models generated by Algorithm \ref{algo1} for solving Problem \eqref{eq:prob-iid}.
			Let $\eta = \sqrt{n K (F(\bw^0)-F^*) / 2}$ and $\etabar = 1/\sqrt{\sigma^2 LT}K$.
			Suppose that Assumptions \ref{as:smooth}--\ref{as:bounded-delay} hold and 
			\begin{align}
				\etabar \leq \frac{1}{4\sqrt{3}LK}
				~\text{and}~
				\eta \etabar \leq 
				\frac{1}{4LK{\lambda}}.
				\label{eq:lr}
			\end{align}
			Then, it holds for $T \geq 1$ that
			\begin{align}
				\frac{1}{T} \sum_{t=0}^{T-1} \E \| \nabla F(\bw^t) \|_2^2
				&\leq \sqrt{\frac{128(F(\bw^0)-F^*)}{n K T}} + \frac{8L}{K T}.
				\label{eq:thm1}
			\end{align} 
		\end{theorem}
		
		The proof of Theorem \ref{thm:iid} is also deferred to the supplementary materials.
		Note that condition \eqref{eq:lr} implies a lower bound on the required number of global rounds
		$T \geq \Omega(n K \lambda^2)$, which exhibits a relaxed dependence on $\lambda$ compared to the lower bound implied by Theorem \ref{thm:hetero}.
		Inequality \eqref{eq:thm1} suggests that the average gradient norm converges to zero at a rate of $\mathcal{O}({1}/{\sqrt{nKT}} + 1/KT)$ as $T$ is sufficiently large. In other words, the total number of communication rounds required scales as $\mathcal{O}(1/nK\epsilon^2+1/K\epsilon)$ so that $\frac{1}{T} \sum_{t=0}^{T-1} \E \| \nabla F(\bw^t) \|_2^2 \leq
		\epsilon$, which also indicates linear speedup wrt to $n$.
		
		When $\lambda = 1$, indicating no delay in the updates, Theorem \ref{thm:iid} implies a rate of $\mathcal{O}(1/\sqrt{nKT} + n/T)$ provided that $T = \Omega(n K)$, which aligns with the convergence rate of synchronous FedAvg presented in \cite{stich2020error}. 
		
		When $K=1$, indicating no local training, the convergence bound \eqref{eq:thm1} reduces to the result for AsySG (an asynchronous SGD algorithm for IID data) \cite{lian2015asynchronous} up to a high-order term. This result demonstrates that the asynchronous SGD converges at a rate of $\mathcal{O}({1}/{\sqrt{nT}})$ when $T=\Omega(n {\lambda}^2)$.
		Besides, it is worth noting that condition \eqref{eq:thm1} implies that $K = \mathcal{O}(T/n{\lambda}^2)$.
		Consequently, if the number of local steps $K$ is not excessively large, the communication complexity decreases linearly as $K$ increases. 
		This highlights the advantage of local training in DeFedAvg-IID, as it asymptotically reduces the number of communication rounds by a factor of $1/K$ while maintaining the same sample complexity as asynchronous SGD.
		
		Furthermore, we notice that there is a similar asynchronous FedAvg algorithm mentioned in \cite[Section 5]{stich2019local} for solving Problem \eqref{eq:prob-iid} with strongly convex objectives. Different from our DeFedAvg-IID, the algorithm in \cite{stich2019local} has varying local steps $K$ across clients, the local learning rates diminish to 0, and there is no global learning rate. 
		The algorithm in \cite{stich2019local} attains $\mathcal{O}(1/nT)$ convergence rate, while this result is limited to strongly convex objectives under the bounded gradient assumption.
		In contrast, Theorem \ref{thm:iid} applies to general nonconvex problems without requiring bounded gradients. This sheds light on the distributed training of nonconvex deep neural networks, showcasing the applicability and effectiveness of DeFedAvg-IID in the more challenging setting.
		
		\begin{figure*}
			\centering
			\includegraphics[width=\textwidth]{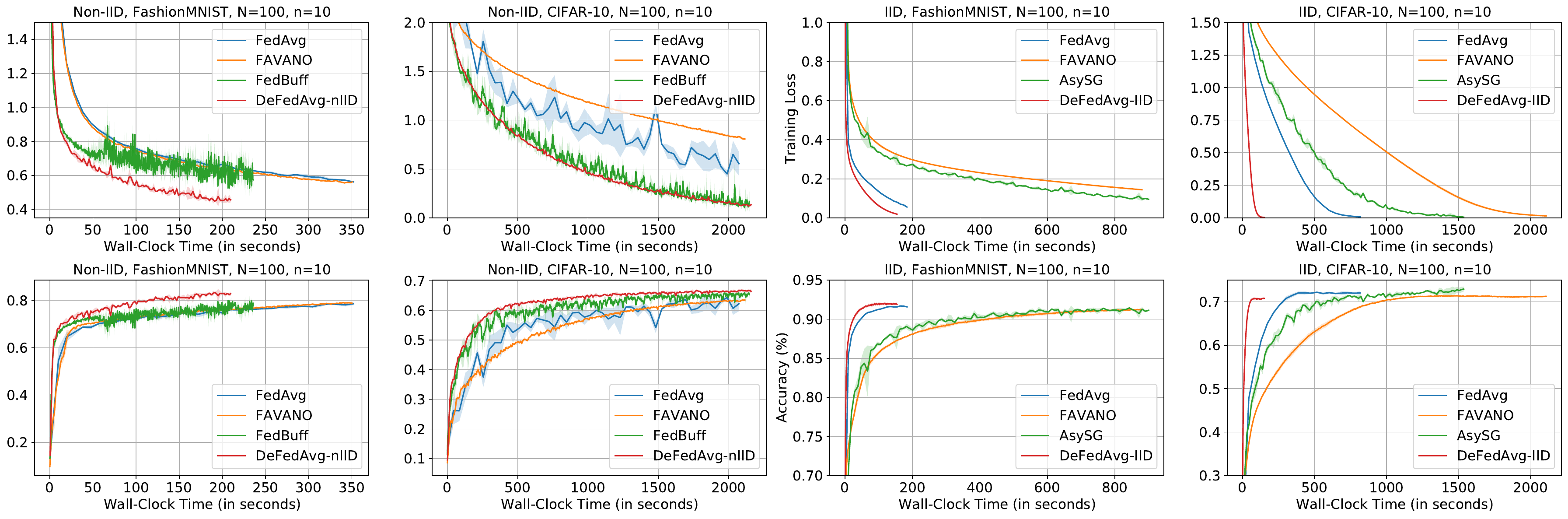}
			\vspace{-5mm}
			\caption{Convergence over wall-clock time of DeFedAvg and other algorithms with $n=10$.}
			\vspace{-2mm}
			\label{fig:acc_loss0.1}
		\end{figure*}
		
		\begin{figure*}[ht]
			\centering
			\includegraphics[width=\textwidth]{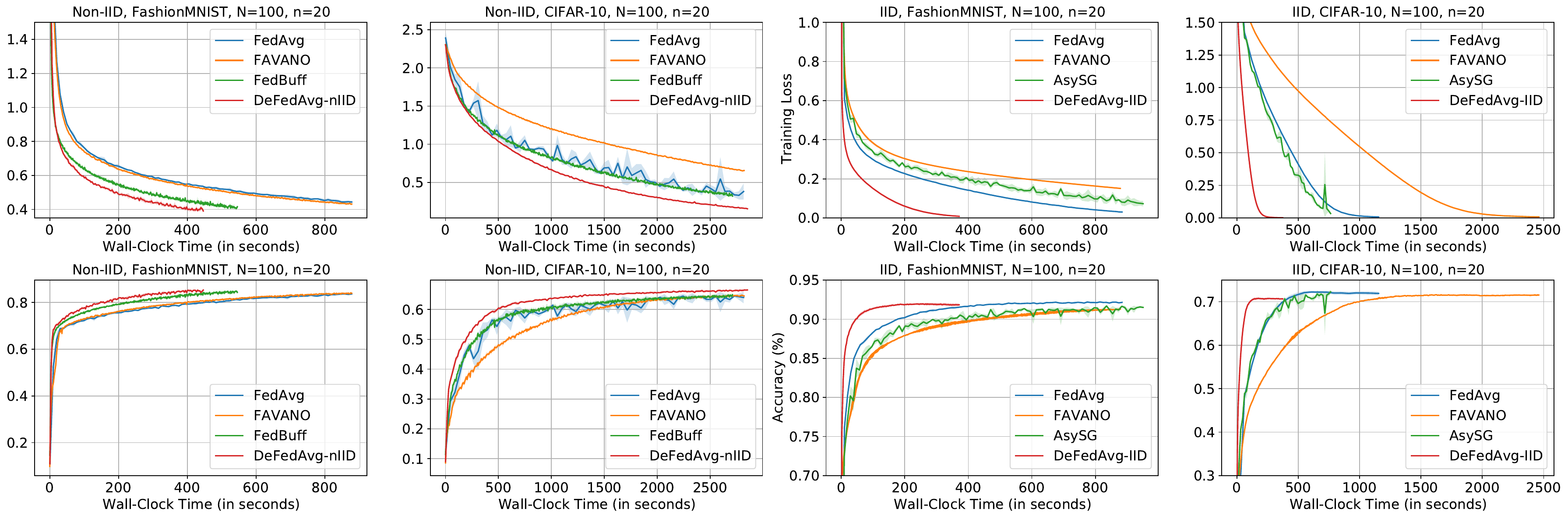}
			\vspace{-5mm}
			\caption{Convergence over wall-clock time of DeFedAvg and other algorithms with $n=20$.
			}
			\vspace{-2mm}
			\label{fig:convergence}
		\end{figure*}
		
		\begin{figure*}[ht]
			\centering
			\includegraphics[width=\textwidth]{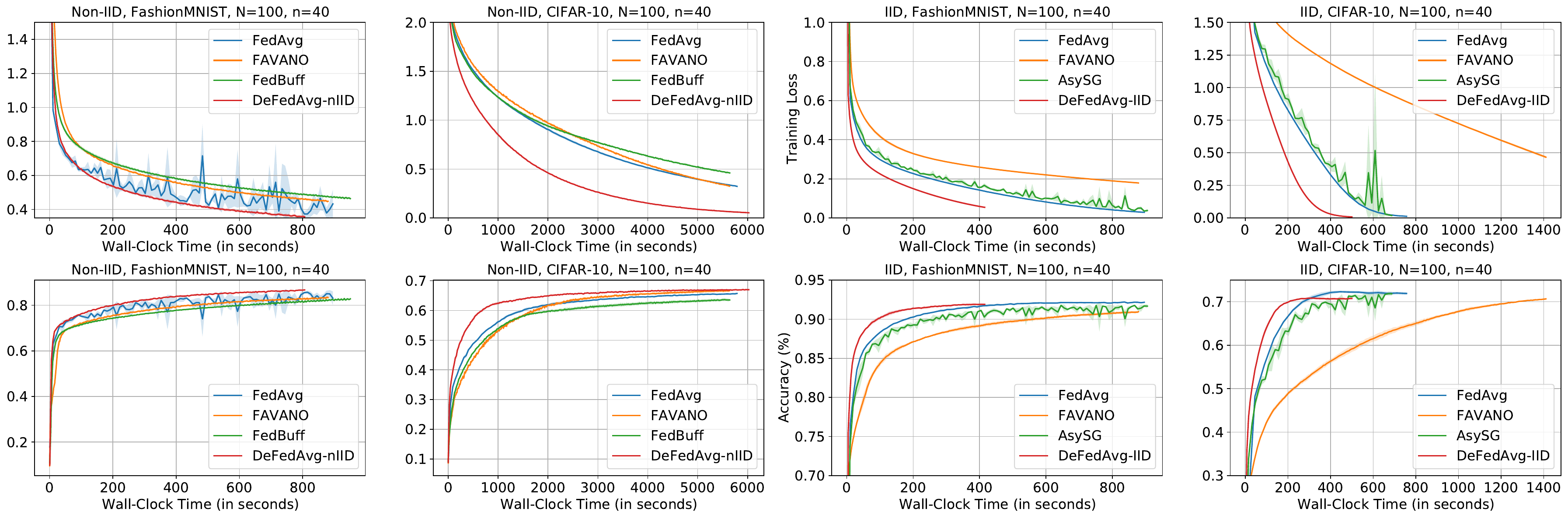}
			\vspace{-5mm}
			\caption{Convergence over wall-clock time of DeFedAvg and other algorithms with  $n=40$.}
			\label{fig:acc_loss0.4}
		\end{figure*}
		
		\begin{figure*}[ht]
			\centering
			\includegraphics[width=\textwidth]{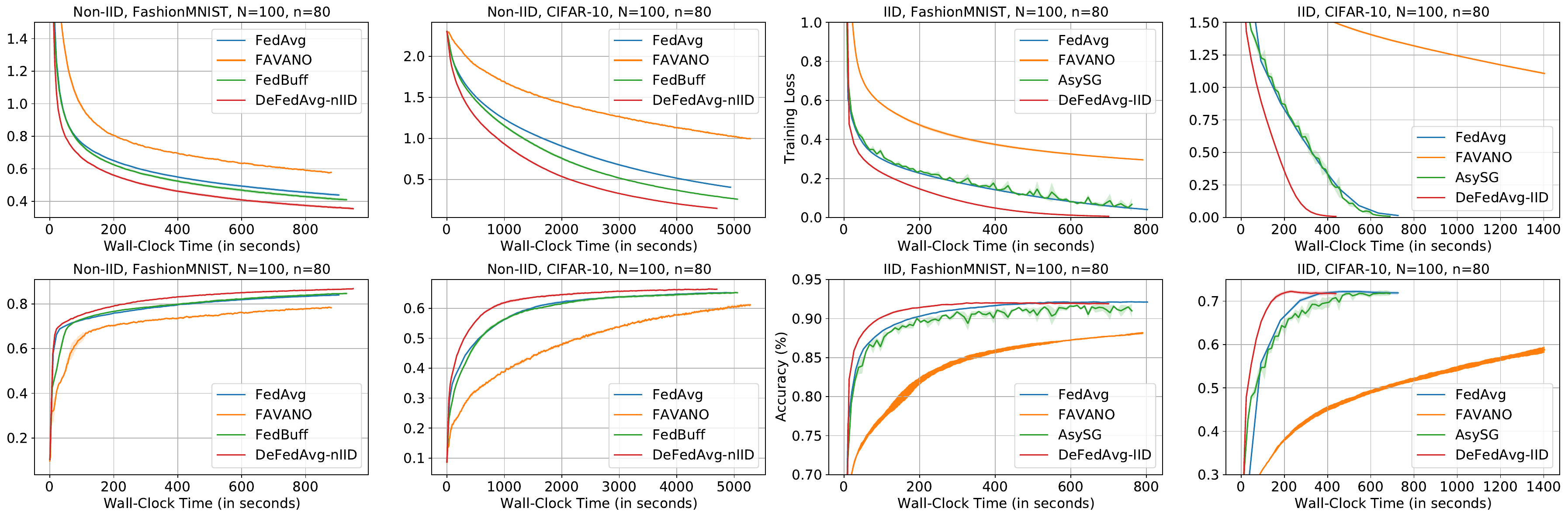}
			\vspace{-5mm}
			\caption{Convergence over wall-clock time of DeFedAvg and other algorithms with $n=80$.}
			\vspace{-2mm}
			\label{fig:acc_loss0.8}
		\end{figure*}
		
		\section{Experiments}		
		\subsection{Experimental Setup}\label{experiments}
		To validate the efficacy of our proposed method, we apply DeFedAvg to train deep neural networks with two convolutional layers and two fully connected layers under the FL scenario. 
		Specifically, we simulate the collaborative training process of $N=100$ clients with diverse processing capabilities, and evaluate the algorithm performance on image classification tasks using the FashionMNIST \cite{fashionmnist} and CIFAR-10 \cite{cifar10} datasets.
		
		Under the non-IID setting, we implement the (synchronous) FedAvg \eqref{eq:localsgd1}--\eqref{eq:localsgd2}, FedBuff \cite{nguyen2022federated}, and the recent AFL algorithm FAVANO \cite{leconte2023favas} for comparisons with DeFedAvg-nIID. 
		To introduce statistical heterogeneity in the non-IID setting, we randomly assign two classes of data to each client. This distribution of data among clients helps create a scenario where the training data across clients differ, thereby increasing the challenge of achieving convergence in the FL process.
		
		Under the IID setting, we implement the FedAvg given by \eqref{eq:localsgd1}--\eqref{eq:localsgd2}, FAVANO \cite{leconte2023favas}, and AsySG \cite{lian2015asynchronous} for comparisons with DeFedAvg-IID.
		Each client is equipped with training data containing all classes. This ensures that the data distribution across clients is statistically homogeneous.
		
		\subsubsection{Neural Network Architecture}
		We adopt a convolutional neural network, which contains two $5 \times 5$ convolutional layers followed by $2 \times 2$ max pooling layers and two fully connected layers with the latent dimensions being $1600$ and $512$, respectively.
		
		\subsubsection{Evaluation Metrics}
		The direct evaluation metric for assessing convergence performance is the number of communication rounds required to reach a target optimization accuracy. However, it is important to note that in scenarios where the computation or communication speeds of different clients vary significantly, the actual runtime of a global round can differ considerably across different algorithms.
		In light of this, it is more reasonable to evaluate the practical performance using wall-clock time, which encompasses both computation time and communication time. 
		This approach provides a more comprehensive assessment of the overall efficiency of the algorithms in real-world scenarios.
		
		\subsubsection{System Model}
		To measure the wall-clock time, we take into account both  the computational time for local training at clients and the communication time for downloading the global model and uploading the local model updates.
		
		Following the approach described in \cite{sun2023semi}, we assume that the average time for local training can be represented as $T_{\text{comp}}=N_{\text{MAC}}/C_{\text{MAC}}$, where $N_{\text{MAC}}$ denotes the number of floating-point operations required for one-iteration training and $C_{\text{MAC}}$ denotes the computation speed of the fastest client, set to be 10GFLOPS. 
		For our experiments, the cost of updating the local models for one iteration using a mini-batch of ten samples is estimated to be 17.0MFLOPs for FashionMNIST and 31.4MFLOPs for CIFAR-10, respectively. 
		Based on these values, we can calculate the computation time required for local training.
		To simulate the hardware heterogeneity, we assume that the computation speeds of clients follow the uniform distribution over interval $[1,5]$, where the slowest client is five times slower than the fastest one.
		This variation in computation speeds allows us to capture the realistic heterogeneity in real-world FL systems.
		
		Furthermore, we assume that the clients transmit the global model and local updates using a  5G network with a transmission bandwidth of 400Mbps for both the downlink and uplink. The model sizes for communication between clients and the server are estimated to be 2.2MB for FashionMNIST and 3.53MB for CIFAR-10, respectively. Based on these values, we can compute the communication time between the server and clients.
		It is worth noting that although we set the same communication time for all clients, the varying computation capabilities of the clients are sufficient to reflect the impact of system heterogeneity. The computational heterogeneity among clients has a significant influence on the overall performance, even when assuming the same communication time for all clients.
		
		\subsubsection{Hyperparameter Selection}
		For all the experiments, the value of the global learning rate $\eta$ (respectively, the local learning rate $\bar{\eta}$) is selected from the set $\{ 0.1, 1.0 \}$ (respectively, $\{ 0.001, 0.005, 0.01, 0.05, 0.1 \}$).
		We report the numerical results based on the learning rates that yield the best convergence performance (note that the AsySG algorithm does not have a local learning rate).
		In all of the experiments, we utilize a total of $N=100$ clients to simulate the FL scenarios.
		Unless otherwise specified, each client performs local iterations with $K=50$ using the mini-batch SGD with a fixed batch size of $10$ in each step.
		We repeat all the experiments with three different random seeds and report the average results.
		
		\begin{table*}[h]
			\centering
			\resizebox{\textwidth}{!}{
				\renewcommand\arraystretch{1.5}
				\begin{tabular}{cccccccccccc}
					\hline \hline
					\multirow{2}{*}{}                &         \multirow{2}{*}{Algorithms}           &    \multirow{2}{*}{Accuracy}      & \multicolumn{4}{c}{FashionMNIST} &  \multirow{2}{*}{Accuracy}  & \multicolumn{4}{c}{CIFAR-10} \\ \cmidrule(lr){4-7} \cmidrule(lr){9-12}
					&         &  & $n=10$   & $n=20$   & $n=40$   & $n=80$ &  & $n=10$  & $n=20$  & $n=40$  & $n=80$ \\ \hline
					\multirow{4}{*}{iid}     & FedAvg   &    \multirow{4}{*}{ $90\%$}   &    51.89      &   179.1     &   184.9     &   179.83     &   \multirow{4}{*}{ $70\%$}    &  302.12     &   411.64    &    282.51   &   287.81   \\
					& FAVANO              &      &  424.25  &   448.89     &    556.25    &  $>$880      &       &   944.84    &    994.19   &   1269.1    &  $>$1400    \\
					& AsySG              &      &  335.16  &   297.57     &    231.22    &    215.38    &       &    501.60   &   369.08    &    400.92   &   345.79   \\
					& DeFedAvg-IID         &  &    \bf{26.39}    &   \bf{49.26}     &    \bf{94.49}    &    \bf{99.72}    &      &  \bf{54.41}     &  \bf{108.52}    &    \bf{207.45}    &  \bf{154.26}    \\ \hline 
					\multirow{4}{*}{non-iid} & FedAvg &  \multirow{4}{*}{ $80\%$}   &  437.57    &   447.81     &    211.67    &    443.56    &  \multirow{4}{*}{ $60\%$}   &    934.11   &  983.46     &    1541.2   & 1495.0     \\
					& FAVANO            &    &   417.21   &     382.01  &    441.85    &    $>$880    &      &   1332.6    &   1374.9   &  1713.3     &  4751.8    \\
					& FedBuff            &   &   198.54    &    218.26    &    604.93    &    422.88    &      &   394.16    &    789.63   &   2008.8    &   1463.5    \\
					& DeFedAvg-nIID       &  &    \bf{103.25}    &   \bf{144.49}     &   \bf{196.98}     &    \bf{236.30}    &       &   \bf{331.60}    &   \bf{497.51}    &   \bf{684.31}    &  \bf{786.85}  \\ \hline \hline
			\end{tabular}}
			\caption{Wall-clock time (in seconds) to achieve the target accuracies with $N=100$ and different values of $n$. 
			}
			\vspace{-2mm}
			\label{tab:main_experiments}
		\end{table*}
		
		\subsection{Numerical Results}
		The experiment results and their associated discussions are provided in the following.
		
		\subsubsection{Comparisons of Convergence Performance}
		Fig. \ref{fig:acc_loss0.1} depicts the convergence of DeFedAvg and its counterparts over wall-clock time for $n=10$. It is observed that DeFedAvg consistently demonstrates the most rapid decrease in training losses and increase in test accuracies.
		In the non-IID case, FedAvg faces similar challenges due to the presence of straggler clients. 
		On the other hand, while FedBuff allows for asynchronous training, it appears to be more prone to weak generalization. This could be because its training process heavily relies on fast clients, while the datasets at slower clients make relatively smaller contributions.
		By contrast, DeFedAvg-nIID enables uniform client sampling, which helps mitigate the problem by ensuring a balanced contribution from all clients. Additionally, when combined with delayed local training, DeFedAvg-nIID achieves faster convergence compared to other competitors.
		In the IID case, FedAvg is dragged down by the straggler clients due to the synchronization in each round, while AsySG necessitates frequent communication between clients and the server in the absence of local training.
		In comparison, the performance of DeFedAvg-IID is enhanced by leveraging both asynchronous updates and local iterations.
		
		In addition to Fig. \ref{fig:acc_loss0.1}, we provide numerical results of the compared algorithms for $n=20,40$, and $80$, as shown in Figures \ref{fig:convergence}, \ref{fig:acc_loss0.4}, and \ref{fig:acc_loss0.8}, respectively.
		The results demonstrate that our proposed DeFedAvg algorithms are scalable as the number of participating clients $n$ increases. 
		In all the experiments, DeFedAvg exhibits superior performance in terms of training losses and test accuracies.
		
		We also compare the wall-clock time of the considered algorithms to reach the target test accuracies for different numbers of participating clients $n$.
		As reported in Table \ref{tab:main_experiments}, DeFedAvg uses substantially less time than other algorithms in all cases.
		For instance, in the IID setting with $n=10$ on CIFAR-10, the time used for DeFedAvg-IID to reach $70\%$ accuracy is almost 1/5 of that for FedAvg, demonstrating the effectiveness of asynchronous updates in DeFedAvg-IID. In the non-IID settings, FedBuff spends less time than the FedAvg when $n$ is small, while the advantage diminishes as $n$ increases. 
		
		\subsubsection{Impact of $n$ in DeFedAvg}
		We explore the impact of the number of participating clients $n$ on the convergence behavior of our proposed DeFedAvg algorithms. 
		In principle, increasing $n$ enhances the system parallelism (leading to less communication rounds) while reduces the client asynchronism (incurring less waiting time). On the contrary, decreasing $n$ reduces the system parallelism while enhances the client asynchronism.
		The results, as depicted in Fig. \ref{fig:client_num}, demonstrate that increasing the number of participating clients expedites the convergence rates of DeFedAvg, which supports the convergence results presented in Theorems \ref{thm:hetero} and \ref{thm:iid}.
		The reason behind this phenomenon is that a larger value of $n$ allows for a more accurate estimation of the full gradient during global updates, thereby facilitating faster convergence over the communication rounds.
		Conversely, when $n$ is small, there is greater variance in the aggregated stochastic gradients, resulting in slower convergence rates and more oscillations in the performance curves.  
		
		\subsubsection{Impact of $K$ in DeFedAvg} 
		We conduct performance testing of DeFedAvg with varying values of local steps $K$. Specifically, we select $K=20, 50, 100$, and the convergence results over communication rounds are presented in Fig. \ref{fig:local_step}.
		As the value of $K$ increases, we observe that DeFedAvg-nIID demonstrates faster convergence rates. This behavior can potentially be attributed to the non-dominant term in the convergence bound \eqref{eq:thm2}, which has an order of $\mathcal{O}(1/KT)$ and is inversely proportional to $K$.
		Likewise, DeFedAvg-IID displays a similar pattern when $K$ increases, thereby supporting the discussion in Sections \ref{sec:thm-homo} regarding the impact of local steps on convergence.
		Moreover, we notice a decrease in the test accuracies on the CIFAR-10 dataset when $K=100$. This observation suggests that increasing the number of local steps may not always result in improved generalization performance.
		
		\begin{figure*}
			\centering
			\includegraphics[width=\textwidth]{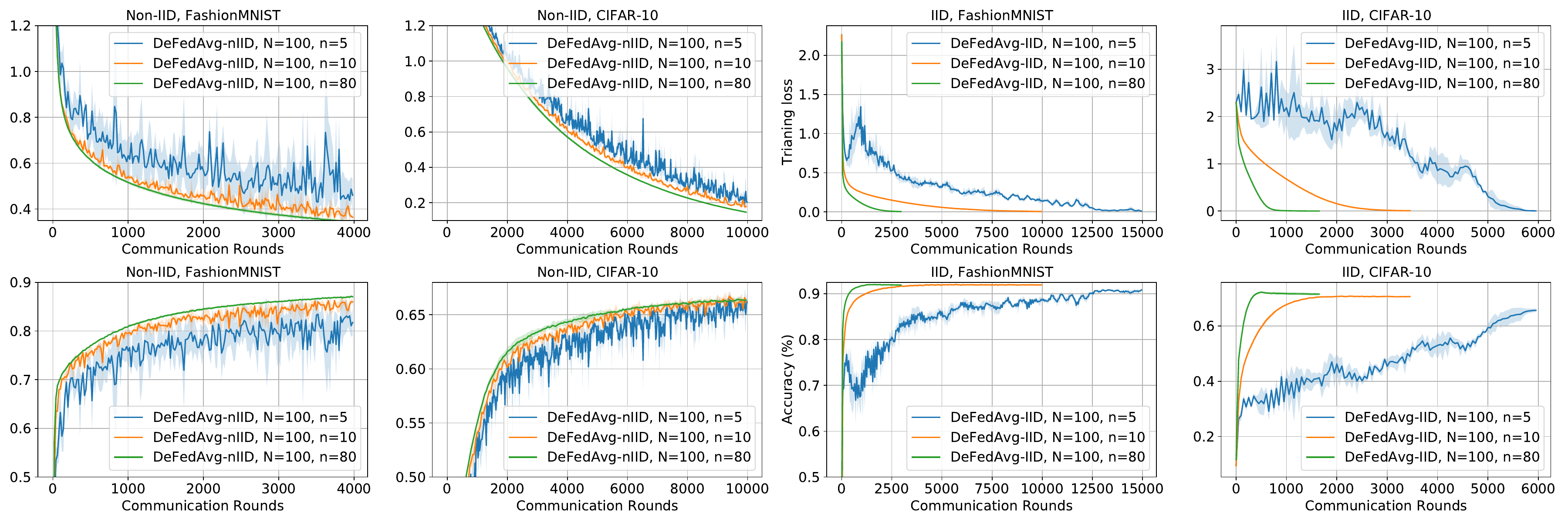}
			\vspace{-5mm}
			\caption{Convergence over communication rounds of DeFedAvg with different values of $n$.}
			\vspace{-2mm}
			\label{fig:client_num}
		\end{figure*}
		
		\begin{figure*}[ht]
			\centering
			\includegraphics[width=\textwidth]{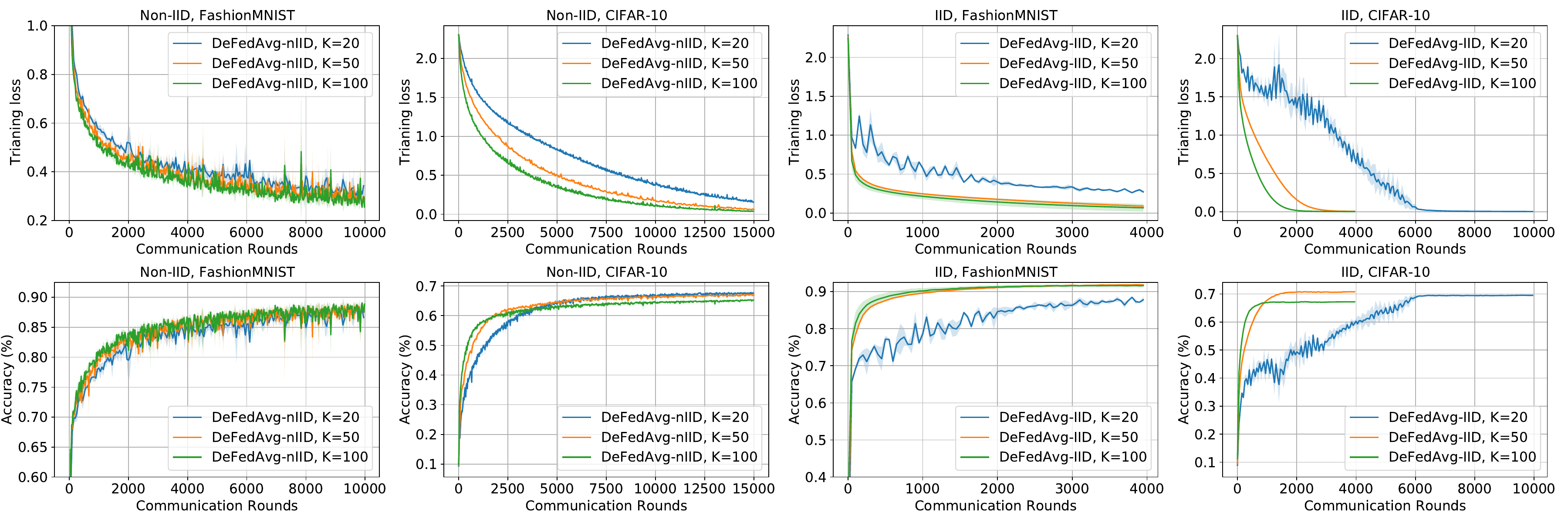}
			\vspace{-5mm}
			\caption{Convergence performance of DeFedAvg with different values of local steps $K$ ($N=100$ and $n=10$).}
			\vspace{-2mm}
			\label{fig:local_step}
		\end{figure*}
		
		\section{Conclusion}
		This paper introduces the DeFedAvg framework to address the straggler effect in FedAvg under system heterogeneity, by enabling asynchronous local training while utilizing receive/send buffers across clients. 
		Theoretical analyses demonstrate that DeFedAvg convergences at asymptotic rates comparable to its synchronous counterparts, and first provably achieves the linear speedup property in AFL, indicating its scalability potential. 
		The algorithmic and theoretical frameworks presented in this paper hold promise for application and extension to other FL problems. For example, they are likely applicable in addressing challenges in client-edge-cloud hierarchical FL \cite{liu2022hierarchical}, client dropout problems \cite{sun2023mimic,shao2022dres}, semi-decentralized AFL \cite{sun2023semi}, coded FL \cite{sun2023stochastic}, federated knowledge distillation \cite{shao2024selective}, over-the-air FL \cite{sun2023channel}, etc. By leveraging the insights from DeFedAvg, these methods can benefit from efficient asynchronous training, potentially improving their performance and scalability.

		\bibliographystyle{IEEEtran}
		\bibliography{DeFedAvg_arXiv}
		
		\vspace{-12mm}
		\begin{IEEEbiography}[{\includegraphics[width=1in,height=1.25in,clip,keepaspectratio]{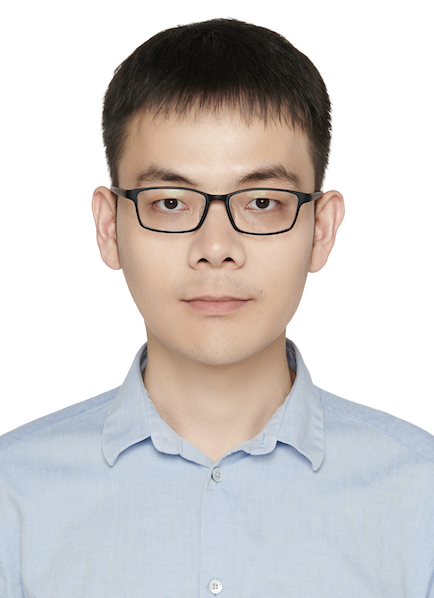}}]{Xiaolu Wang} received the B.Eng. degree in telecommunications engineering and the M.Sc. degree in communication and information systems from Xidian University, Xi’an, China, in 2014 and 2017, respectively, and the Ph.D. degree in systems engineering and engineering management (SEEM) from The Chinese University of Hong Kong (CUHK), Hong Kong, China, in 2022. He was a Research Associate with the Department of SEEM at CUHK. He is currently a Postdoctoral Fellow with the Department of Electronic and Computer Engineering, Hong Kong University of Science and Technology (HKUST), Hong Kong, China. His research interests include optimization, data science, and networked systems.
		\end{IEEEbiography}
		
		\vspace{-12mm}
		\begin{IEEEbiography}[{\includegraphics[width=1in,height=1.25in,clip,keepaspectratio]{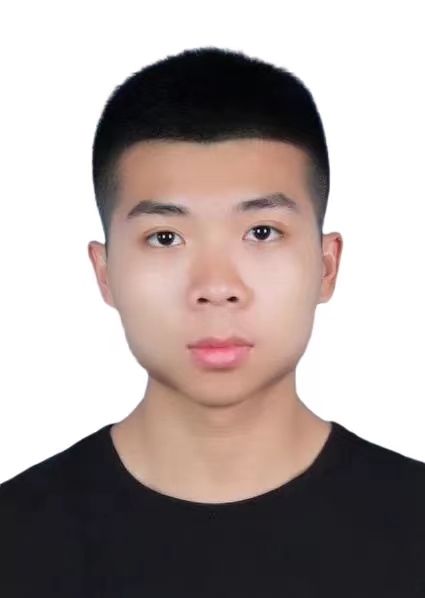}}]{Zijian Li} (Graduate Student Member, IEEE) received the B.Eng. degree in Electrical Engineering and Automation from the South China University of Technology in 2020, and the M.Sc. degree in electronic and information engineering from the Hong Kong Polytechnic University in 2022. He is currently pursuing a Ph.D. degree in the Department of Electronic and Computer Engineering at the Hong Kong University of Science and Technology. His research interest is federated learning.
		\end{IEEEbiography}
		
		\vspace{-12mm}
		\begin{IEEEbiography}[{\includegraphics[width=1in,height=1.25in,clip,keepaspectratio]{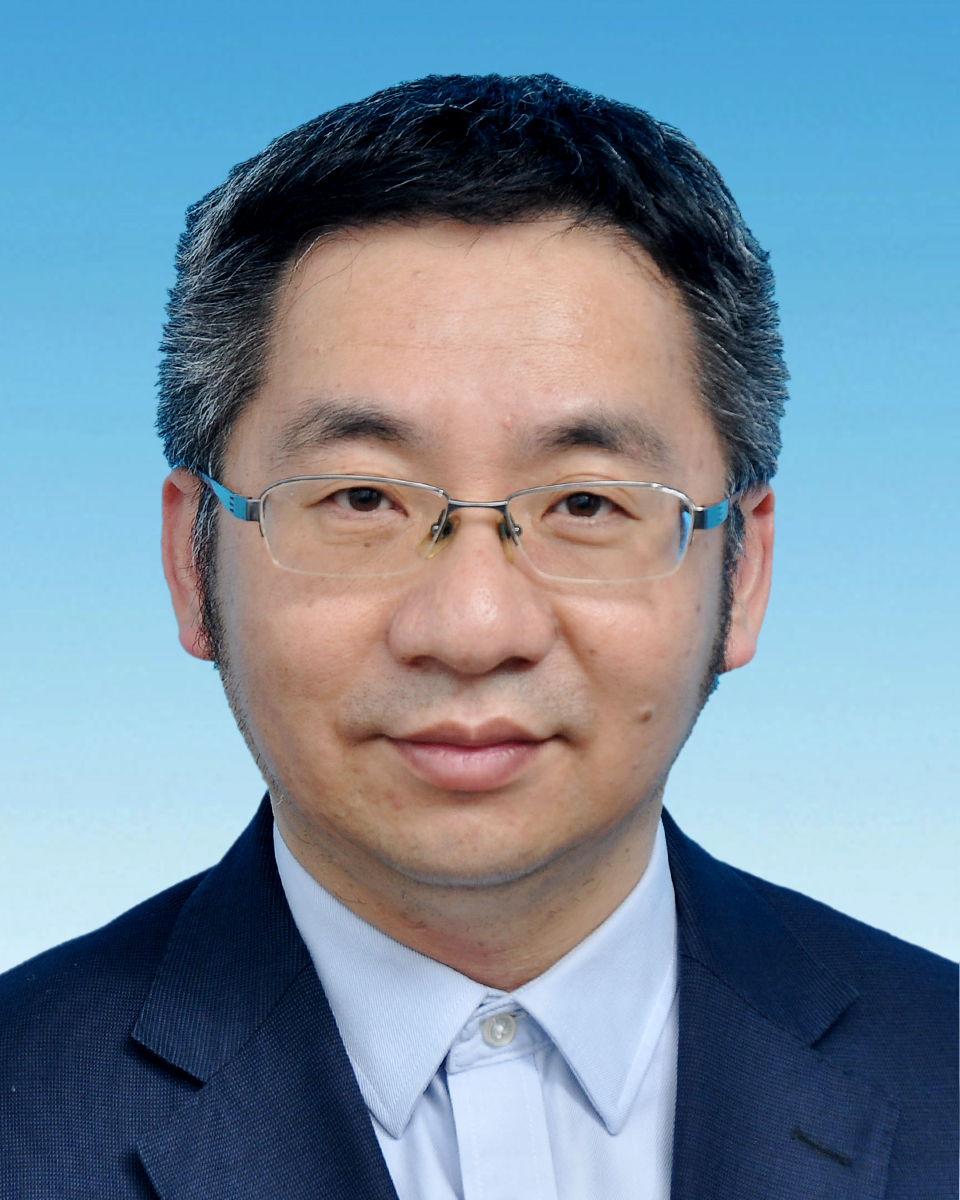}}]{Shi Jin} (Fellow, IEEE) received the B.S. degree in communications engineering from Guilin University of Electronic Technology, Guilin, China, in 1996, the M.S. degree from Nanjing University of Posts and Telecommunications, Nanjing, China, in 2003, and the Ph.D. degree in information and communications engineering from the Southeast University, Nanjing, in 2007. From June 2007 to October 2009, he was a Research Fellow with the Adastral Park Research Campus, University College London, London, U.K. He is currently with the faculty of the National Mobile Communications Research Laboratory, Southeast University. His research interests include wireless communications, random matrix theory, and information theory. He is serving as an Area Editor for the Transactions on Communications and IET Electronics Letters. He was an Associate Editor for the IEEE Transactions on Wireless Communications, and IEEE Communications Letters, and IET Communications. Dr. Jin and his co-authors have been awarded the 2011 IEEE Communications Society Stephen O. Rice Prize Paper Award in the field of communication theory, the IEEE Vehicular Technology Society 2023 Jack Neubauer Memorial Award, a 2022 Best Paper Award and a 2010 Young Author Best Paper Award by the IEEE Signal Processing Society.
		\end{IEEEbiography}
		
		\vspace{-12mm}
		\begin{IEEEbiography}[{\includegraphics[width=1in,height=1.25in,clip,keepaspectratio]{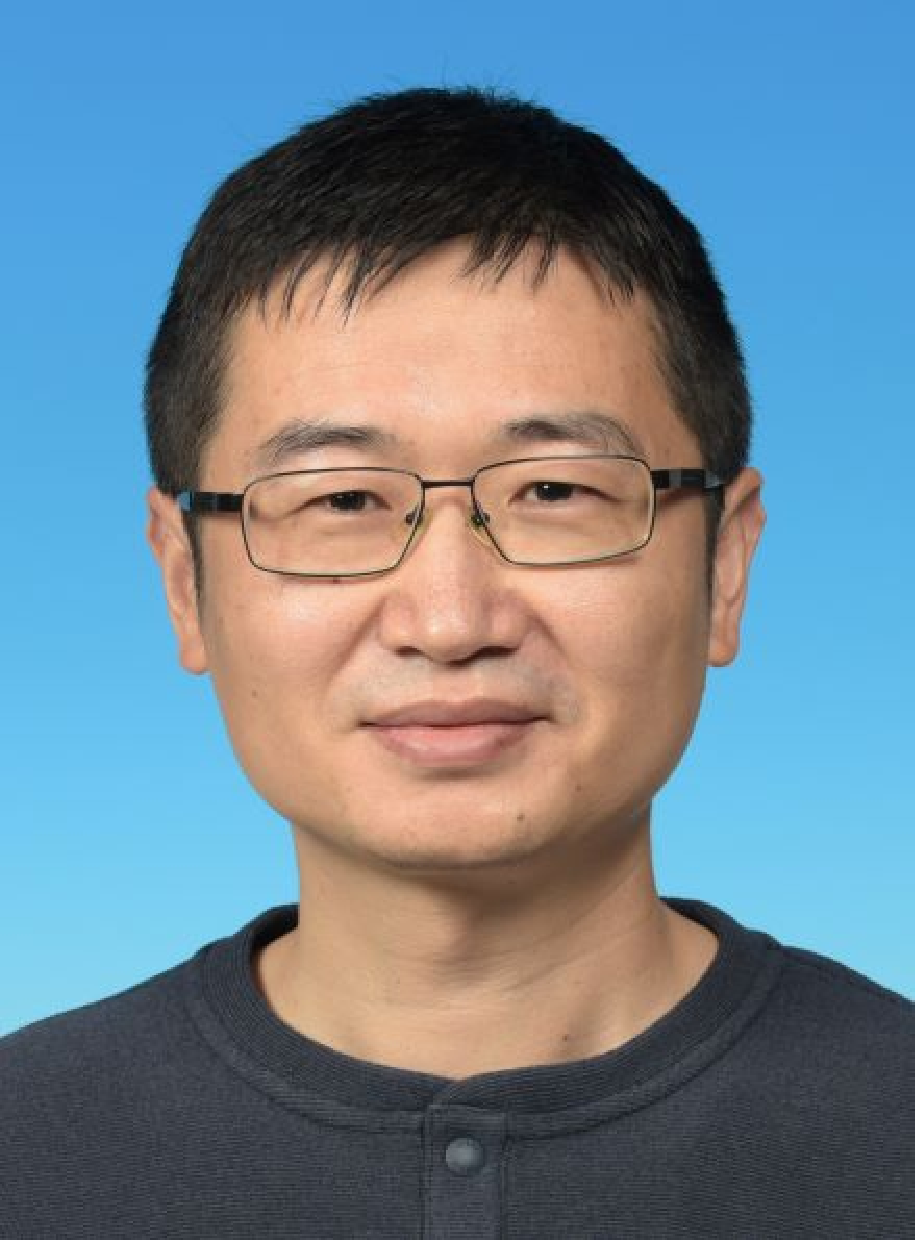}}]{Jun Zhang} (Fellow, IEEE) received the B.Eng. degree in Electronic Engineering from the University of Science and Technology of China in 2004, the M.Phil. degree in Information Engineering from the Chinese University of Hong Kong in 2006, and the Ph.D. degree in Electrical and Computer Engineering from the University of Texas at Austin in 2009. He is an Associate Professor in the Department of Electronic and Computer Engineering at the Hong Kong University of Science and Technology. His research interests include wireless communications and networking, mobile edge computing and edge AI, and cooperative AI. Dr. Zhang co-authored the book Fundamentals of LTE (Prentice-Hall, 2010). He is a co-recipient of several best paper awards, including the 2021 Best Survey Paper Award of the IEEE Communications Society, the 2019 IEEE Communications Society \& Information Theory Society Joint Paper Award, and the 2016 Marconi Prize Paper Award in Wireless Communications. Two papers he co-authored received the Young Author Best Paper Award of the IEEE Signal Processing Society in 2016 and 2018, respectively. He also received the 2016 IEEE ComSoc Asia-Pacific Best Young Researcher Award. He is an Editor of IEEE Transactions on Communications, IEEE Transactions on Machine Learning in Communications and Networking, and was an editor of IEEE Transactions on Wireless Communications (2015-2020). He served as a MAC track co-chair for IEEE Wireless Communications and Networking Conference (WCNC) 2011 and a co-chair for the Wireless Communications Symposium of IEEE International Conference on Communications (ICC) 2021. He is an IEEE Fellow and an IEEE ComSoc Distinguished Lecturer.
		\end{IEEEbiography}
		
		\clearpage
		\appendices
		\section{Proofs of Theorems}
		In this appendix, we provide the proofs of Theorems \ref{thm:hetero} and \ref{thm:iid}. To begin with, we introduce two facts that will be frequently used.
		
		\begin{fact}\label{fact:expectation}
			Let $U$ and $V$ be two random variables, and $\mathcal{F}$ be a sigma algebra. If $U$ is $\mathcal{F}$-measurable, then we have
			\[
			\E[UV]  = \E [U \E [V \vert \mathcal{F}]].
			\]
		\end{fact}
		\begin{proof}
			Since $U$ is $\mathcal{F}$-measurable, we have 
			\[
			\E[UV \vert \mathcal{F}] = U \E [V \vert \mathcal{F}].
			\]
			Then, it follows from the law of total expectation that 
			\[
			\E[UV]  = \E[\E [UV \vert \mathcal{F}]] = \E [U \E [V \vert \mathcal{F}]],
			\]
			as desired.
		\end{proof}
		
		\begin{fact}\label{fact:sum}
			Suppose that $m\geq2$ is an integer and $\bx_1,\dots,\bx_m$ are vectors in the same inner product space. Then, the following inequality holds:
			\[
			\left\| \sum_{i=1}^{m} \bx_i \right\|_2^2 \leq m \sum_{i=1}^{m} \|\bx_i\|_2^2.
			\]
		\end{fact}
		For notational convenience in the sequel, we define $\bG_i^t \coloneqq \bG_i^t$, $\bG_i^{t,k} \coloneqq \nabla F_i ( \bw_i^{\tauit,k} )$, and $\bg_i^{t,k} \coloneqq f ( \bw_i^{\tauit,k}; \bxi_i^{t,k} )$ for $i \in [N]$ and $t,k \in \mathbb{N}$. We use $\E_{X}[\cdot]$ to denote the conditional expectation by taking expectation wrt $X$ while holding other variables constant. 
		
		\subsection{Proof of Theorem \ref{thm:hetero}}\label{sec:proof2}
		
		\subsubsection{Technical Lemmas}
		\begin{lemma}\label{lem:inner-prod-hetero}
			Suppose that Assumptions \ref{as:smooth} and \ref{as:unbias} hold. Then, it hold for all $t \geq 0$ that
			\begin{align*}
				&\E \langle \nabla F(\bw^t), \etabar K \nabla F(\bw^t) - \bm{g}_{\textnormal{nIID}}^t \rangle
				\nonumber
				\\
				\leq& \frac{K}{2} \etabar \E \| \nabla F(\bw^t) \|_2^2 
				+ \frac{L^2}{N} \etabar \sum_{i=1}^{N} \sum_{k=0}^{K-1} \E \| \bw_i^{\tauit,k} - \bw^{\tauit} \|_2^2
				\nonumber
				\\
				&\hspace{-1mm}+\hspace{-1mm} \frac{L^2 K}{N} \etabar \sum_{i=1}^{N} \E \| \bw^{\tauit} \hspace{-0.5mm}-\hspace{-0.5mm} \bw^t \|_2^2
				\hspace{-0.5mm}-\hspace{-0.5mm} \frac{1}{2K} \etabar \E \left\| \frac{1}{N} \hspace{-0.5mm}\sum_{i=1}^{N}\hspace{-0.5mm} \sum_{k=0}^{K-1} \hspace{-0.5mm}\bG_i^{t,k} \right\|_2^2\hspace{-0.5mm}.
			\end{align*}
		\end{lemma}
		\begin{proof}
			Using the local iteration formula \eqref{eq:local-iter'}, we have 
			$
			\bm{\Delta}_i^t = \bw_i^{\tauit,0} - \bw_i^{\tauit,K} = \sum_{k=0}^{K-1} \etabar \bg_i^{t,k},
			$
			which implies that
			\begin{align}
				\bm{g}_{\textnormal{nIID}}^t = \frac{1}{|\mathcal{I}_t|} \sum_{i \in \mathcal{I}_t} \bm{\Delta}_i^t = \frac{\etabar}{n} \sum_{i \in \mathcal{I}_t} \sum_{k=0}^{K-1} \bg_i^{t,k}.
				\label{eq:ht}
			\end{align}
			Then, we have 
			\begin{align}
				&\E_{\mathcal{I}_t} [\bm{g}_{\textnormal{nIID}}^t]
				= \E_{\mathcal{I}_t} \left[ \frac{1}{n} \sum_{i\in\mathcal{I}_t} \bm{\Delta}_i^t \right]
				= \E_{i_1^t,\dots,i_n^t} \left[ \frac{1}{n} \sum_{m=1}^{n} \bm{\Delta}_{i_m^t}^t \right]
				\nonumber
				\\
				&\stackrel{(a)}{=} \hspace{-0.5mm}\frac{1}{n} \hspace{-0.5mm} \left( \E_{i_1^t} [\bm{\Delta}_{i_1^t}^t] + \dots + \E_{i_n^t} [\bm{\Delta}_{i_n^t}^t] \right)
				\hspace{-0.5mm}\stackrel{(b)}{=}\hspace{-0.5mm} \E_{i_1^t} [\bm{\Delta}_{i_1^t}^t]
				\hspace{-0.5mm}\stackrel{(c)}{=}\hspace{-0.5mm} \sum_{i=1}^{N} \frac{1}{N} \bm{\Delta}_i^t
				\nonumber
				\\
				&= \frac{1}{N} \sum_{i=1}^{N} \sum_{k=0}^{K-1} \etabar \bg_i^{t,k},
				\label{eq:unbias}
			\end{align}
			where $(a)$ and $(b)$ are because $i_1^t, \dots,i_n^t$ are independent and identically distributed, respectively, $(c)$ is implies by the uniform distribution of $i_1^t$. 
			We remark that the terms $\sum_{k=0}^{K-1} \etabar \nabla f ( \bw_i^{\tauit,k}; \bxi_i^{t,k} )$ in \eqref{eq:unbias} with $i \notin \mathcal{I}_t$ are \textit{virtual}. They are introduced for mathematical convenience, although the calculations associated with these terms may not occur in practical systems.
			Further, it follows from Fact \ref{fact:expectation} that
			\begin{align}
				& \E \langle \nabla F(\bw^t),  \etabar K \nabla F(\bw^t) - \bm{g}_{\textnormal{nIID}}^t \rangle
				\nonumber
				\\
				=&~ \etabar \E \left\langle \nabla F(\bw^t), K \nabla F(\bw^t) - \frac{1}{N} \sum_{i=1}^{N} \sum_{k=0}^{K-1} \bg_i^{t,k} \right\rangle
				\nonumber
				\\
				\stackrel{(a)}{=}& \etabar \E \hspace{-0.5mm}\left\langle\hspace{-1.5mm} \sqrt{K} \nabla F(\bw^t), 
				\frac{1}{\sqrt{K} n} \hspace{-1mm}\sum_{i=1}^{N} \hspace{-1mm} \left( \hspace{-1.2mm}K \nabla F_i(\bw^t) \hspace{-0.8mm}-\hspace{-1.5mm} \sum_{k=0}^{K-1} \hspace{-1mm}\E_{\bxi_i^{t,k}} \hspace{-1mm}\left[\hspace{-0.5mm} \bg_i^{t,k} \right] \hspace{-1mm}\right) \hspace{-1.5mm}\right\rangle\hspace{-1mm},
				\nonumber
				\\
				\stackrel{(b)}{=}& \etabar \E \hspace{-0.5mm}\left\langle\hspace{-0.5mm} \sqrt{K} \nabla F(\bw^t), 
				\frac{1}{\sqrt{K} n} \sum_{i=1}^{N} \hspace{-0.5mm} \left( \hspace{-0.5mm}K \nabla F_i(\bw^t) \hspace{-0.5mm}-\hspace{-0.5mm} \sum_{k=0}^{K-1} \bG_i^{t,k} \hspace{-0.5mm}\right) \hspace{-0.5mm}\right\rangle,
				\nonumber
				\\
				\stackrel{(c)}{=}& \frac{K}{2} \etabar \E \| \nabla F(\bw^t) \|_2^2 
				- \frac{1}{2K} \etabar \E \left\| \frac{1}{N} \sum_{i=1}^{N} \sum_{k=0}^{K-1} \bG_i^{t,k} \right\|_2^2
				\nonumber
				\\
				& + \frac{1}{2K} \etabar \underbrace{\E \left\| \frac{1}{N} \sum_{i=1}^{N} \sum_{k=0}^{K-1} \left( \bG_i^{t,k} - \nabla F_i(\bw^t) \right) \right\|_2^2}_{X^t},
				\label{eq:inner-prod'}
			\end{align}
			where $(a)$ is implied by Fact \ref{fact:expectation}, $(b)$ follows from Assumption \ref{as:unbias}, and $(c)$ uses the identity $\langle \bx, \by \rangle = (\| \bx \|_2^2 + \| \by \|_2^2 -  \| \bx - \by \|_2^2)/2$ for vectors $\bx$ and $\by$.
			To upper bound $X^t$, we have
			\begin{align}
				X^t 
				\stackrel{(a)}{\leq}& 2 \E \left\| \frac{1}{N} \sum_{i=1}^{N} \sum_{k=0}^{K-1} \left( \bG_i^{t,k} - \nabla F_i ( \bw^{\tauit} ) \right) \right\|_2^2
				\nonumber
				\\
				&+ 2 \E \left\| \frac{1}{N} \sum_{i=1}^{N} \sum_{k=0}^{K-1} \left( \nabla F_i ( \bw^{\tauit} ) - \nabla F_i(\bw^t) \right) \right\|_2^2
				\nonumber
				\\
				\stackrel{(b)}{\leq}& \frac{2}{N^2} \sum_{i=1}^{N} \sum_{k=0}^{K-1} N K \E \| \bG_i^{t,k} - \nabla F_i(\bw^{\tauit}) \|_2^2
				\nonumber
				\\
				&+ \frac{2}{N^2} \sum_{i=1}^{N} \sum_{k=0}^{K-1} N K \E \| \nabla F_i ( \bw^{\tauit} ) - \nabla F_i(\bw^t) \|_2^2
				\nonumber
				\\
				\stackrel{(c)}{\leq}& \frac{2 L^2 K}{N} \sum_{i=1}^{N} \sum_{k=0}^{K-1} \E \| \bw_i^{\tauit,k} - \bw^{\tauit} \|_2^2
				\nonumber
				\\
				&+ \frac{2 L^2 K^2}{N} \sum_{i=1}^{N} \E \| \bw^{\tauit} - \bw^t \|_2^2,
				\label{eq:Xt'}
			\end{align}
			where $(a)$ and $(b)$ follow from Fact \ref{fact:sum} with $m = 2$ and $m=n K$, respectively, and $(c)$ is implied by Assumption \ref{as:smooth}. Plugging \eqref{eq:Xt'} back into \eqref{eq:inner-prod'} gives the desired result.
		\end{proof}
		
		\begin{lemma}\label{lem:sumsum'}
			Suppose that Assumptions \ref{as:unbias}, \ref{as:bounded-var}, and \ref{as:bounded-delay} hold, and $\mathcal{I}_t = \{j_1^t,\dots,j_n^t\}$ is a multiset whose elements are independently and uniformly sampled from $[N]$ with replacement in round $t$ ($t \geq 0$) of Algorithm \ref{algo2}. Then, it holds for all $t \geq 0$ that
			\begin{align}
				\E \left\| \sum_{j \in \mathcal{I}_t} \sum_{k=0}^{K-1} ( \bg_j^{t,k} - \bG_j^{t,k} ) \right\|_2^2
				&\leq 2 n K \sigma^2,
				\label{eq:lema-heter} 
			\end{align}
			Besides, it holds for all $i \in [n]$ and $t \geq 1$ that
			\begin{align}
				\E \left\| \sum_{s=\tauit}^{t-1} \sum_{j \in \mathcal{I}_s} \sum_{k=0}^{K-1} ( \bg_j^{s,k} - \bG_j^{s,k} ) \right\|_2^2 
				&\leq 2 n K {\lambda} \sigma^2.
				\label{eq:lemb-heter}
			\end{align}
		\end{lemma}
		\begin{proof}
			For notational convenience, we define 
			$\bm{\phi}_j^t \coloneqq  \sum_{k=0}^{K-1} ( \bg_j^{t,k} - \bG_j^{t,k} )$
			for $t \geq 1$ and $j \in \mathcal{I}_t$.
			Then, we have
			\begin{align} 
				&\E_{\mathcal{I}_t} \left\| \sum_{j \in \mathcal{I}_t} \bm{\phi}_j^t \right\|_2^2
				= \E_{\mathcal{I}_t} \left[ \sum_{i,j \in \mathcal{I}_t} \left\langle \bm{\phi}_i^t, \bm{\phi}_j^t \right\rangle \right]
				\nonumber
				\\
				&= \E_{j^t_1,\dots,j^t_n} \left[ \sum_{u,v\in[n]} \left\langle \bm{\phi}_{j^t_u}^t, \bm{\phi}_{j^t_v}^t \right\rangle \right]
				\nonumber
				\\
				& \stackrel{(a)}{=} \sum_{u=1}^{n} \E_{j^t_u} \| \bm{\phi}_{j^t_u}^t \|_2^2
				+ \sum_{\substack{u,v\in[n], u \neq v}} \E_{j^t_u,j^t_v} [ \langle \bm{\phi}_{j^t_u}^t, \bm{\phi}_{j^t_v}^t \rangle ]
				\nonumber
				\\
				& \stackrel{(b)}{=} \sum_{u=1}^{n} \sum_{j=1}^{N} \frac{1}{N} \| \bm{\phi}_j^t \|_2^2
				+ \sum_{\substack{u,v\in[n], u \neq v}} \sum_{i,j \in [N]} \frac{1}{N^2} \langle \bm{\phi}_{i}^t, \bm{\phi}_{j}^t \rangle 
				\nonumber
				\\
				& = \frac{n}{N} \sum_{j=1}^{N} \| \bm{\phi}_j^t \|_2^2
				+ \frac{n(n-1)}{N^2} \sum_{i,j \in [N]} \langle \bm{\phi}_i^t, \bm{\phi}_j^t \rangle 
				\label{eq:temp}
				\\
				& =\hspace{-0.5mm} \frac{n(N\hspace{-0.5mm}+\hspace{-0.5mm}n\hspace{-0.5mm}-\hspace{-0.5mm}1)}{N^2} \sum_{j=1}^{N} \| \bm{\phi}_j^t \|_2^2
				\hspace{-0.5mm}+\hspace{-0.5mm}\frac{n(n\hspace{-0.5mm}-\hspace{-0.5mm}1)}{N^2} \hspace{-2mm}\sum_{\substack{i,j \in [N], i \neq j}}\hspace{-2mm} \langle \bm{\phi}_i^t, \bm{\phi}_j^t \rangle,
				\label{eq:sumi'}
			\end{align}
			where $(a)$ and $(b)$ hold because $j_1,\dots,j_n$ are independent and uniformly distributed over $[N]$, respectively. 
			For any integers $i, j \in {I}_t$ such that $i \neq j$ and $k,\ell \in [0,K-1]$, we have 
			$
			\E{} \left\langle {} \bg_i^{t,k} {}-{} \bG_i^{t,k}{}, 
			\bg_j^{t,\ell} {}-{} \bG_j^{t,\ell} \right\rangle
			{}={} \E {}\left\langle{} \bg_i^{t,k} {}-{} \bG_i^{t,k}{}, 
			\E_{\bxi_i^{t,\ell}}{} \left[ \bg_j^{t,\ell} {}-{} \bG_j^{t,\ell} \right] \right\rangle
			{}={} 0,
			\nonumber
			$
			where the first equality uses Fact \ref{fact:expectation} and the second equality follows from Assumption \ref{as:unbias}.
			Therefore,
			\begin{align}
				\E \left\langle \bm{\phi}_i^t, \bm{\phi}_j^t \right\rangle
				\hspace{-0.5mm}=\hspace{-3mm} \sum_{0 \leq k,\ell \leq K-1} \hspace{-3mm} \E \hspace{-0.5mm}\left\langle \bg_i^{t,k} \hspace{-0.5mm}-\hspace{-0.5mm} \bG_i^{t,k}, 
				\bg_j^{t,\ell} \hspace{-0.5mm}-\hspace{-0.5mm} \bG_j^{t,\ell} \right\rangle
				= 0,
				\label{eq:crossterm=0}
			\end{align}
			which implies that 
			$$
			\E \left[ \sum_{\substack{i,j \in [N], i \neq j}} \langle \bm{\phi}_i^t, \bm{\phi}_j^t \rangle \right] = \sum_{\substack{i,j \in [N],i \neq j}} \E \left\langle \bm{\phi}_i^t, \bm{\phi}_j^t \right\rangle = 0.
			$$
			It then follows from the law of total expectation and \eqref{eq:sumi'} that
			\begin{align}
				\hspace{-2mm}\E \hspace{-0.5mm}\left\| \sum_{j \in \mathcal{I}_t}\hspace{-1.5mm} \bm{\phi}_j^t \hspace{-0.5mm}\right\|_2^2
				\hspace{-2.5mm}=\hspace{-1mm} \E\hspace{-1mm} \left[\hspace{-0.5mm} \E_{\mathcal{I}_t} \hspace{-1mm}\left\| \sum_{j \in \mathcal{I}_t} \hspace{-1.5mm}\bm{\phi}_j^t \right\|_2^2 \right]
				\hspace{-1.5mm}=\hspace{-1mm} \E \hspace{-1mm}\left[\hspace{-0.5mm} \frac{n(\hspace{-0.5mm}N\hspace{-1mm}+\hspace{-1mm}n\hspace{-1mm}-\hspace{-1mm}1\hspace{-0.5mm})}{N^2} \hspace{-1mm}\sum_{j=1}^{N}\hspace{-0.5mm} \| \bm{\phi}_j^t \|_2^2 \hspace{-0.5mm}\right]\hspace{-1mm}.\hspace{-0.5mm}
				\label{eq:sdd}
			\end{align}
			To proceed, we note that for any integers $k,\ell \in [0,K-1]$ such that $k < \ell$, it follows from Fact \ref{fact:expectation} and Assumption \ref{as:unbias} that 
			$
			\E \left\langle \bg_j^{t,k} - \bG_j^{t,k}, 
			\bg_j^{t,\ell} - \bG_j^{t,\ell} \right\rangle
			=\E \left\langle \bg_j^{t,k} - \bG_j^{t,k}, 
			\E_{\bxi_i^{t,\ell}} \left[ \bg_j^{t,\ell} - \bG_j^{t,\ell} \right]\right\rangle
			= 0.
			$                    
			This implies that
			\begin{align}
				\E \| \bm{\phi}_j^t \|_2^2 \hspace{-0.5mm}=\hspace{-0.5mm} \E\hspace{-0.5mm} \left\| \sum_{k=0}^{K-1} \hspace{-1.5mm}\left(\hspace{-0.5mm} \bg_j^{t,k} \hspace{-1.5mm}-\hspace{-0.5mm} \bG_j^{t,k} \hspace{-0.5mm}\right)\hspace{-0.5mm} \right\|_2^2
				\hspace{-1.5mm}=\hspace{-0.5mm} \sum_{k=0}^{K-1} \hspace{-0.5mm}\E\hspace{-0.5mm} \left\| \bg_j^{t,k} \hspace{-1.5mm}-\hspace{-0.5mm} \bG_j^{t,k} \right\|_2^2\hspace{-0.5mm}.\hspace{-0.5mm}
				\label{eq:sumk'}
			\end{align}
			Substituting \eqref{eq:sumk'} into \eqref{eq:sdd} yields
			\begin{align}
				&\E \left\| \sum_{j \in \mathcal{I}_t} \bm{\phi}_j^t \right\|_2^2
				= \frac{n(N+n-1)}{N^2} \sum_{j=1}^{N} \sum_{k=0}^{K-1} \E \| \bg_j^{t,k} - \bG_j^{t,k} \|_2^2
				\nonumber
				\\
				&\stackrel{(a)}{\leq} \frac{n(N+n-1)}{N^2} N K \sigma^2
				\stackrel{(b)}{\leq} 2 n K \sigma^2,
				\label{eq:lema-heter'}
			\end{align}
			where $(a)$ uses Assumption \ref{as:bounded-var} and $(b)$ holds because $n \leq N$. This completes the proof of \eqref{eq:lema-heter}.
			
			We further note that for any integers $s,r \in [\tauit, t-1]$ such that $s < r$, it follows from Fact \ref{fact:expectation} that
			\begin{align}
				&\E_{\mathcal{I}_r} \left\langle \sum_{j \in \mathcal{I}_s} \bm{\phi}_j^s, 
				\sum_{j \in \mathcal{I}_r} \bm{\phi}_j^r \right\rangle
				= \E_{j_1^r,\dots,j_n^r} \left\langle \sum_{j \in \mathcal{I}_s} \bm{\phi}_j^s, 
				\sum_{u=1}^{n} \bm{\phi}_{j_u^r}^r \right\rangle
				\nonumber
				\\
				&=\left\langle \sum_{j \in \mathcal{I}_s} \bm{\phi}_j^s, 
				\sum_{u=1}^{n}  \E_{j_u^r} [ \bm{\phi}_{j_u^r}^r ] \right\rangle
				= \left\langle \sum_{j \in \mathcal{I}_s} \bm{\phi}_j^s, 
				\sum_{u=1}^{n}  \sum_{j=1}^{N} \frac{1}{N} \bm{\phi}_j^{r} \right\rangle
				\nonumber
				\\
				&= \frac{n}{N} \left\langle \sum_{j \in \mathcal{I}_s} \bm{\phi}_j^s, 
				\sum_{j=1}^{N} \bm{\phi}_j^{r} \right\rangle.
				\label{eq:e}
			\end{align}
			Taking expectation wrt $\bxi^r \coloneqq \{ \bxi_j^{r,k}: j \in [N], k = 0,1,\dots,K-1 \}$ in \eqref{eq:e} gives
			\begin{align}
				&\hspace{-2mm}\E_{\bxi^r} \left\langle \sum_{j \in \mathcal{I}_s} \bm{\phi}_j^s, 
				\sum_{j=1}^{N} \bm{\phi}_j^{r} \right\rangle
				\nonumber
				\\
				&\hspace{-2mm}=\hspace{-0.5mm}\left\langle\hspace{-0.2mm} \sum_{j \in \mathcal{I}_s} \hspace{-1mm}\sum_{k=0}^{K-1}\hspace{-0.5mm} (\bg_j^{s,k} \hspace{-1.5mm}-\hspace{-0.5mm} \bG_j^{s,k}), 
				\sum_{j=1}^{N} \hspace{-0.5mm}\sum_{k=0}^{K-1} \hspace{-1mm}\E_{\bxi_j^{r,k}} \hspace{-1.5mm}\left[ \bg_j^{r,k} \hspace{-1.5mm}-\hspace{-0.5mm} \bG_j^{r,k} \hspace{-0.5mm}\right] \hspace{-1mm}\right\rangle
				\hspace{-1mm}=\hspace{-0.5mm} 0,\hspace{-0.5mm}
				\label{eq:ee}
			\end{align}
			where the second equality uses Fact \ref{fact:expectation} and the last equality uses Assumption \ref{as:unbias}.
			Combining \eqref{eq:e} and \eqref{eq:ee} implies that for any integers $s,r \in [\tauit, t-1]$ such that $s < r$, we have
			\begin{align}
				&\E \left\langle \sum_{j \in \mathcal{I}_s} \bm{\phi}_j^s, 
				\sum_{j \in \mathcal{I}_r} \bm{\phi}_j^r \right\rangle
				= \E \left[ \E_{\mathcal{I}_r} \left\langle \sum_{j \in \mathcal{I}_s} \bm{\phi}_j^s, 
				\sum_{j \in \mathcal{I}_r} \bm{\phi}_j^r \right\rangle \right] 
				\nonumber
				\\
				&\hspace{-1mm}=\hspace{-1mm} \frac{n}{N} \E \hspace{-1mm}\left[\hspace{-0.5mm} \left\langle \sum_{j \in \mathcal{I}_s} \hspace{-1.5mm} \bm{\phi}_j^s, \sum_{j=1}^{N}\hspace{-1mm} \bm{\phi}_j^{r}\hspace{-1mm} \right\rangle\hspace{-0.5mm} \right]
				\hspace{-2mm}=\hspace{-1mm} \frac{n}{N} \E \hspace{-1mm}\left[\hspace{-0.5mm} \E_{\bxi^r}\hspace{-1mm} \left\langle \sum_{j \in \mathcal{I}_s} \hspace{-1mm}\bm{\phi}_j^s, \sum_{j=1}^{N} \hspace{-1mm}\bm{\phi}_j^{r} \hspace{-1mm}\right\rangle \hspace{-0.5mm}\right]
				\hspace{-1mm}=\hspace{-0.5mm} 0.
				\nonumber
			\end{align}
			It follows that
			\begin{align}
				&\E \left\| \sum_{s=\tauit}^{t-1} \sum_{j \in \mathcal{I}_s} \bm{\phi}_j^s \right\|_2^2 
				= \sum_{\tauit \leq s,r \leq t-1} \E \left\langle \sum_{j \in \mathcal{I}_s} \bm{\phi}_j^s, 
				\sum_{j \in \mathcal{I}_r} \bm{\phi}_j^r \right\rangle
				\nonumber
				\\
				&= \sum_{s=\tauit}^{t-1} \E \left\| \sum_{j \in \mathcal{I}_s} \bm{\phi}_j^s \right\|_2^2
				\leq 2 n K {\lambda} \sigma^2,
			\end{align}
			where the inequality follows from \eqref{eq:lema-heter'} and Assumption \ref{as:bounded-delay}.  This completes the proof of \eqref{eq:lemb-heter}.
		\end{proof}

		\begin{lemma}\label{lem:grad-var'}
			Suppose that Assumptions \ref{as:unbias}, \ref{as:bounded-var}, \ref{as:bounded-delay}, and \ref{as:bounded-grad} hold. Then, it hold for all $t \geq 0$ that
			\begin{align*}
				\E \| \bm{g}_{\textnormal{nIID}}^t \|_2^2
				\hspace{-0.5mm}\leq\hspace{-0.5mm} \frac{4 K \sigma^2 + 2 G^2 K^2}{n} \etabar^2
				\hspace{-0.5mm}+\hspace{-0.5mm} 2 \etabar^2 \E \left\| \frac{1}{N} \sum_{i=1}^{N} \sum_{k=0}^{K-1} \bG_i^{t,k} \right\|_2^2.
				\nonumber
			\end{align*}
		\end{lemma}
		\begin{proof}
			In view of \eqref{eq:ht}, we have
			\begin{align}
				&\E \| \bm{g}_{\textnormal{nIID}}^t \|_2^2 
				= \E \left\| \frac{1}{n} \sum_{i \in \mathcal{I}_t} \sum_{k=0}^{K-1} \etabar \bg_i^{t,k} \right\|_2^2 
				\nonumber
				\\
				&= \E \left\| \frac{1}{n} \sum_{i \in \mathcal{I}_t} \sum_{k=0}^{K-1} \etabar \left( \bg_i^{t,k} - \bG_i^{t,k} + \bG_i^{t,k} \right) \right\|_2^2
				\nonumber 
				\\
				&\stackrel{(a)}{\leq} \hspace{-1mm} 2 \E \left\| \hspace{-0.5mm}\frac{1}{n}\hspace{-0.5mm} \sum_{i \in \mathcal{I}_t} \hspace{-0.5mm}\sum_{k=0}^{K-1} \hspace{-0.5mm}\etabar \left( \bg_i^{t,k} \hspace{-1.5mm}-\hspace{-0.5mm} \bG_i^{t,k} \right) \right\|_2^2
				\hspace{-2mm}+\hspace{-0.5mm} 2 \E \left\| \hspace{-0.5mm}\frac{1}{n}\hspace{-0.5mm} \sum_{i \in \mathcal{I}_t} \hspace{-0.5mm}\sum_{k=0}^{K-1} \hspace{-0.5mm}\etabar \bG_i^{t,k} \right\|_2^2
				\nonumber
				\\
				&\stackrel{(b)}{\leq} \frac{2\etabar^2}{n^2} 2 n K \sigma^2 
				+ \frac{2 \etabar^2}{n^2} \E \left\| \sum_{i \in \mathcal{I}_t} \sum_{k=0}^{K-1} \bG_i^{t,k} \right\|_2^2.
				\label{eq:Delta-temp'}
			\end{align}
			where the $(a)$ uses Fact \ref{fact:sum} with $m=2$, $(b)$ follows from Lemma \ref{lem:sumsum'}.  
			Following similar arguments for showing \eqref{eq:temp}, we can obtain
			\begin{align}
				&\E_{\mathcal{I}_t} \left\| \sum_{i \in \mathcal{I}_t} \sum_{k=0}^{K-1} \bG_i^{t,k} \right\|_2^2
				\nonumber
				\\ 
				=& \frac{n}{N} \hspace{-0.5mm}\sum_{i=1}^{N}\hspace{-0.5mm} \left\| \sum_{k=0}^{K-1}\hspace{-0.5mm} \bG_i^{t,k} \right\|_2^2
				\hspace{-1mm}+\hspace{-0.5mm} \frac{n(n\hspace{-0.5mm}-\hspace{-1mm}1)}{N^2}\hspace{-2mm} \sum_{i,j \in [N]} \hspace{-1mm} \left\langle \hspace{-0.5mm}\sum_{k=0}^{K-1}\hspace{-0.5mm} \bG_i^{t,k}, \hspace{-0.5mm}\sum_{k=0}^{K-1}\hspace{-0.5mm} \bG_j^{t,k} \hspace{-1mm}\right\rangle 
				\nonumber
				\\
				=& \frac{n}{N} \sum_{i=1}^{N} \left\| \sum_{k=0}^{K-1} \bG_i^{t,k} \right\|_2^2
				+ \frac{n(n-1)}{N^2} \left\| \sum_{i=1}^{N} \sum_{k=0}^{K-1} \bG_i^{t,k} \right\|_2^2.
				\label{eq:condexp} 
			\end{align}
			Taking total expectation for both sides of \eqref{eq:condexp} yields
			\begin{align}
				& \E \left\| \sum_{i \in \mathcal{I}_t} \sum_{k=0}^{K-1} \bG_i^{t,k} \right\|_2^2
				= \E \left[ \E_{\mathcal{I}_t} \left\| \sum_{i \in \mathcal{I}_t} \sum_{k=0}^{K-1} \bG_i^{t,k} \right\|_2^2 \right]
				\nonumber
				\\
				&= \hspace{-0.5mm}\frac{n}{N}\hspace{-0.5mm} \sum_{i=1}^{N} \E \left\| \hspace{-0.5mm}\sum_{k=0}^{K-1}\hspace{-0.5mm} \bG_i^{t,k} \right\|_2^2
				\hspace{-1mm}+\hspace{-0.5mm} \frac{n(n\hspace{-0.5mm}-\hspace{-0.5mm}1)}{N^2} \E \left\| \sum_{i=1}^{N} \hspace{-0.5mm}\sum_{k=0}^{K-1}\hspace{-0.5mm} \bG_i^{t,k} \right\|_2^2.
				\label{eq:Esumsum}
			\end{align}
			Plugging \eqref{eq:Esumsum} back into \eqref{eq:Delta-temp'} gives
			\begin{align}
				&\E \| \bm{g}_{\textnormal{nIID}}^t \|_2^2
				\leq 4 K \etabar^2 \frac{\sigma^2}{n} 
				+ \frac{2\etabar^2}{nN} \sum_{i=1}^{N} \E \left\| \sum_{k=0}^{K-1} \bG_i^{t,k} \right\|_2^2
				\nonumber
				\\
				&\qquad + \frac{2(n-1)\etabar^2}{n N^2} \E \left\| \sum_{i=1}^{N} \sum_{k=0}^{K-1} \bG_i^{t,k} \right\|_2^2
				\nonumber
				\\
				&\leq 4 K \etabar^2 \frac{\sigma^2}{n} 
				+ \frac{2\etabar^2}{nN} \sum_{i=1}^{N} \sum_{k=0}^{K-1} K \E \| \bG_i^{t,k} \|_2^2
				\nonumber
				\\
				&\qquad + \frac{2\etabar^2}{N^2} \E \left\| \sum_{i=1}^{N} \sum_{k=0}^{K-1} \bG_i^{t,k} \right\|_2^2
				\nonumber
				\\
				&\leq 4 K \etabar^2 \frac{\sigma^2}{n} 
				\hspace{-0.5mm}+\hspace{-0.5mm} \frac{2 G^2 K^2 \etabar^2}{n}
				\hspace{-0.5mm}+\hspace{-0.5mm} 2 \etabar^2 \E \left\| \frac{1}{N} \sum_{i=1}^{N} \sum_{k=0}^{K-1} \bG_i^{t,k} \right\|_2^2,
				\label{eq:Delta-temp''}
			\end{align}
			where the second inequality follows from Fact \ref{fact:sum} with $m=K$ and the third inequality holds due to Assumption \ref{as:bounded-grad}.
			This completes the proof. 
		\end{proof}

		\begin{lemma}\label{lem:local-iter-diff‘}
			Suppose that Assumptions \ref{as:smooth}, \ref{as:unbias} \ref{as:bounded-var}, and \ref{as:bounded-hetero} hold. Then, it holds for all $t \geq 0$ that
			\begin{align*}
				&\frac{1}{N} \sum_{i=1}^{N} \sum_{k=0}^{K-1} \E \| \bw_i^{\tauit,k} - \bw^{\tauit} \|_2^2
				\leq  (\sigma^2+8K\nu^2) 2 K^2 \etabar^2 
				\nonumber
				\\
				&+ 16 K^3 \etabar^2  \E \| \nabla F(\bw^t) \|_2^2 
				\hspace{-0.5mm}+\hspace{-0.5mm} \frac{16 L^2 K^3}{N} \etabar^2 \sum_{i=1}^{N} \E \| \bw^{\tauit} \hspace{-0.5mm}-\hspace{-0.5mm} \bw^t \|_2^2.
			\end{align*}
			\begin{proof}
				Using the local iteration formula, we have
				\begin{align}
					&\E \| \bw_i^{\tauit,k} - \bw^{\tauit} \|_2^2
					\nonumber
					\\
					=&\E \| \bw_i^{\tauit,k-1} - \etabar \bg_i^{t,k-1} - \bw^{\tauit} \|_2^2
					\nonumber
					\\
					=&\E \Big\| \bw_i^{\tauit,k-1} - \bw^{\tauit} - \etabar \Big( \bG_i^{t,k-1} - \bG_i^t
					+ \bG_i^t
					\nonumber
					\\
					&\quad  - \nabla F_i(\bw^t) + \nabla F_i(\bw^t) - \nabla F(\bw^t) + \nabla F(\bw^t) \Big) 
					\nonumber
					\\
					&\quad - \etabar \Big( \bg_i^{t,k-1} - \bG_i^{t,k-1} \Big) \Big\|_2^2 
					\nonumber
					\\
					=&\E \Big\| \bw_i^{\tauit,k-1} - \bw^{\tauit} - \etabar \Big( \bG_i^{t,k-1} - \bG_i^t + \bG_i^t 
					\nonumber
					\\
					&\quad - \nabla F_i(\bw^t) + \nabla F_i(\bw^t) - \nabla F(\bw^t) + \nabla F(\bw^t) \Big) \Big\|_2^2
					\nonumber
					\\
					&\quad + \etabar^2 \E \| \bg_i^{t,k-1} - \bG_i^{t,k-1} \|_2^2,
					\label{eq:lem:local-iter-diff‘1} 
				\end{align}
				where the last equality holds because Fact \ref{fact:expectation} and Assumption \ref{as:unbias} imply that the cross term is 0, e.g.,
				\begin{align*}
					& \E \Big\langle \bw_i^{\tauit,k-1} - \bw^{\tauit} - \etabar \left( \bG_i^{t,k-1} - \bG_i^t + \bG_i^t - \nabla F_i(\bw^t) \right.
					\nonumber
					\\
					&\quad \left. + \nabla F_i(\bw^t) - \nabla F(\bw^t) + \nabla F(\bw^t) \right), 
					\bg_i^{t,k-1} - \bG_i^{t,k-1} \Big\rangle
					\nonumber
					\\
					&= \E \Big\langle \bw_i^{\tauit,k-1} \hspace{-0.5mm}-\hspace{-0.5mm} \bw^{\tauit} - \etabar \hspace{-0.5mm}\left(\hspace{-0.5mm} \bG_i^{t,k-1} \hspace{-0.5mm}-\hspace{-0.5mm} \bG_i^t \hspace{-0.5mm}+\hspace{-0.5mm} \bG_i^t \hspace{-0.5mm}-\hspace{-0.5mm} \nabla F_i(\bw^t)  \right.
					\nonumber
					\\
					&\qquad \left. + \nabla F_i(\bw^t) - \nabla F(\bw^t) + \nabla F(\bw^t) \right), 
					\nonumber
					\\
					&\qquad \E \left[\bg_i^{t,k-1} - \bG_i^{t,k-1} ~\middle|~ \bw_i^{\tauit,k-1},\bw^t \right] \Big\rangle
					= 0.
				\end{align*}
				It follows that
				\begin{align}
					&\E \| \bw_i^{\tauit,k} - \bw^{\tauit} \|_2^2
					\nonumber
					\\
					\stackrel{(a)}{\leq}&~ \etabar^2 \sigma^2 + \left( 1 + \frac{1}{2K-1}\right) \E \| \bw_i^{\tauit,k-1} - \bw^{\tauit} \|_2^2
					\nonumber
					\\
					&+ 2K \etabar^2 \E \| \bG_i^{t,k-1} - \bG_i^t + \bG_i^t - \nabla F_i(\bw^t) + \nabla F_i(\bw^t) 
					\nonumber
					\\
					&- \nabla F(\bw^t) + \nabla F(\bw^t) \|_2^2
					\nonumber
					\\
					\stackrel{(b)}{\leq}& \left( 1 + \frac{1}{2K-1}\right) \E \| \bw_i^{\tauit,k-1} - \bw^{\tauit} \|_2^2 
					\nonumber
					\\
					&+ 8 K \etabar^2 \E \| \bG_i^{t,k-1} - \bG_i^t \|_2^2 
					+ 8 K \etabar^2 \E \| \bG_i^t - \nabla F_i(\bw^t) \|_2^2
					\nonumber
					\\
					&+ 8 K \etabar^2 \E \| \nabla F_i(\bw^t) - \nabla F(\bw^t) \|_2^2
					\nonumber
					\\
					&+ 8 K \etabar^2  \E \| \nabla F(\bw^t) \|_2^2 + \etabar^2 \sigma^2
					\nonumber
					\\
					\stackrel{(c)}{\leq}& \left( 1 + \frac{1}{2K-1} + 8 L^2 K \etabar^2\right) \E \| \bw_i^{\tauit,k-1} - \bw^{\tauit} \|_2^2
					\nonumber
					\\
					&+ 8 L^2 K \etabar^2 \E \| \bw^{\tauit} - \bw^t \|_2^2 
					+ 8 K \etabar^2 \nu^2
					\nonumber
					\\
					&+ 8 K \etabar^2  \E \| \nabla F(\bw^t) \|_2^2 
					+ \etabar^2 \sigma^2
					\nonumber
					\\
					\stackrel{(d)}{\leq}& \hspace{-0.5mm}\left(\hspace{-0.5mm} 1 \hspace{-0.5mm}+\hspace{-0.5mm} \frac{1}{K\hspace{-0.5mm}-\hspace{-0.5mm}1}\right) \E \| \bw_i^{\tauit,k-1} \hspace{-0.5mm}-\hspace{-0.5mm} \bw^{\tauit} \|_2^2
					+ (\sigma^2+8K\nu^2) \etabar^2
					\nonumber
					\\
					& + 8 K \etabar^2  \E \| \nabla F(\bw^t) \|_2^2 + 8 L^2 K \etabar^2 \E \| \bw^{\tauit} - \bw^t \|_2^2,
					\label{eq:recur'}
				\end{align}
				where $(a)$ uses Assumption \ref{as:bounded-var} and the fact that $\| \bx + \by \|_2^2 \leq (1+1/\beta) \|\bx\|_2^2 + (1+\beta) \|\by\|_2^2$ for vectors $\bx,\by$ and $\beta=2K-1$, $(b)$ uses Fact \ref{fact:sum} with $n=3$, $(c)$ is implied by Assumptions \ref{as:smooth} and \ref{as:bounded-hetero}, and $(d)$ holds because 
				$
				\etabar \leq \frac{1}{4LK}
				\Rightarrow 1 + \frac{1}{2K-1} + 8 L^2 K \etabar^2 \leq 1 + \frac{1}{K-1}.
				$
				Let $Z_i^{t,k} \coloneqq \E \| \bw_i^{\tauit,k} - \bw^{\tauit} \|_2^2$, $a \coloneqq 1 + \frac{1}{K-1}$, and $b \coloneqq (\sigma^2+8K\nu^2) \etabar^2 + 8 K \etabar^2  \E \| \nabla F(\bw^t) \|_2^2 + 8 L^2 K \etabar^2 \E \| \bw^{\tauit} - \bw^t \|_2^2$, then \eqref{eq:recur'} can be written as 
				\[
				Z_i^{t,k} \leq a Z_i^{t,k-1} + b.
				\]
				Solving this recursion gives
				\begin{align}
					&Z_i^{t,k} 
					\leq a^k Z_i^{t,0} + b \sum_{\ell=0}^{k-1} a^\ell 
					= \frac{a^k - 1}{a - 1} b  
					\nonumber
					\\
					&= (K-1) \left( \left( 1+ \frac{1}{K-1} \right)^{K-1} - 1 \right) b
					\leq 2 (K-1) b,
					\nonumber
				\end{align}
				where the second inequality holds due to the fact that $\left( 1+ \frac{1}{K-1} \right)^{K-1} \leq e \leq 3$ with $e$ is being Euler's number.
				This implies that
				\begin{align}
					&\frac{1}{N} \sum_{i=1}^{N} \sum_{k=0}^{K-1} \E \| \bw_i^{\tauit,k} - \bw^{\tauit} \|_2^2
					\leq \frac{1}{N} \sum_{i=1}^{N} \sum_{k=0}^{K-1} 2 (K-1) b.
					\nonumber
				\end{align}
				Substituting $b$ and simplifying gives the desired result.
			\end{proof}
		\end{lemma}
		
		\begin{lemma}\label{lem:global-iter-diff'}
			Suppose that Assumptions \ref{as:unbias}, \ref{as:bounded-var}, and \ref{as:bounded-delay} hold. Then, it holds for all $t \geq 0$ and $i \in \mathcal{I}_t$ that
			\begin{align}
				&\E \| \bw^t - \bw^{\tauit} \|_2^2
				\leq \frac{4 \sigma^2 K {\lambda} + 2 G^2 K^2 {\lambda}^2}{n} \eta^2 \etabar^2
				\nonumber
				\\
				&\qquad + 2 {\lambda} \eta^2 \etabar^2 \sum_{s=(t-{\lambda})_+}^{t-1} \E \left\| \frac{1}{N} \sum_{j=1}^{N} \sum_{k=0}^{K-1} \bG_j^{s,k} \right\|_2^2.
				\nonumber
			\end{align}
		\end{lemma}
		\begin{proof}
			For each $i \in \mathcal{I}_t$, we consider the following two cases: \textsf{i)} $\tauit = t$. Then, we have $\E \| \bw^t - \bw^{\tauit} \|_2^2 = 0$. \textsf{ii)} $\tauit < t$. In view of \eqref{eq:local-iter'} and \eqref{eq:global-iter'}, the server's global iteration in round $s \in [\tauit,t-1]$ can be written as
			$
			\bw^{s+1} = \bw^s - \frac{\eta}{|\mathcal{I}_s|} \sum_{j \in \mathcal{I}_s} \etabar \sum_{k=0}^{K-1} \bg_j^{s,k}. 
			$
			Then, we have
			\begin{align}
				&\E \| \bw^t - \bw^{\tauit} \|_2^2
				=\E \left\| \sum_{s=\tauit}^{t-1} (\bw^{s+1} - \bw^s) \right\|_2^2
				\nonumber
				\\
				&= \E \left\| \sum_{s=\tauit}^{t-1} \frac{\eta}{n} \sum_{j \in \mathcal{I}_s} \sum_{k=0}^{K-1}  \etabar \bg_j^{s,k}  \right\|_2^2
				\nonumber 
				\\
				&\leq 2 \E \left\| \sum_{s=\tauit}^{t-1} \frac{\eta}{n} \sum_{j \in \mathcal{I}_s} \sum_{k=0}^{K-1}  \etabar \left( \bg_j^{s,k} - \bG_j^{s,k} \right) \right\|_2^2
				\nonumber
				\\
				&\quad + 2 \E \left\| \sum_{s=\tauit}^{t-1} \frac{\eta}{n} \sum_{j \in \mathcal{I}_s} \sum_{k=0}^{K-1}  \etabar \bG_j^{s,k} \right\|_2^2
				\nonumber
				\\
				&\stackrel{(a)}{\leq} \hspace{-1mm}\frac{2\eta^2 \etabar^2 }{n^2} 2 n K {\lambda} \sigma^2
				\hspace{-1mm}+\hspace{-1mm} \frac{2\eta^2\etabar^2}{n^2} |t\hspace{-0.5mm}-\hspace{-0.5mm}\tauit| \hspace{-1.5mm}\sum_{s=\tauit}^{t-1}\hspace{-2mm} \E \left\| \hspace{-0.2mm}\sum_{j \in \mathcal{I}_s}\hspace{-1mm} \sum_{k=0}^{K-1}\hspace{-1mm} \bG_j^{s,k} \right\|_2^2
				\nonumber
				\\
				&\stackrel{(b)}{\leq} \hspace{-1mm}4 K {\lambda} \eta^2 \etabar^2 \frac{\sigma^2}{n}
				\hspace{-1mm}+\hspace{-1mm} \frac{2{\lambda}\eta^2\etabar^2}{n^2} \hspace{-2mm}\sum_{s=(t-{\lambda})_+}^{t-1} \hspace{-2mm}\E \left\| \frac{1}{n}\hspace{-0.5mm} \sum_{j \in \mathcal{I}_s}\hspace{-0.5mm} \sum_{k=0}^{K-1}\hspace{-0.5mm} \bG_j^{s,k} \right\|_2^2\hspace{-0.5mm}, \hspace{-0.5mm}
				\label{eq:delayiter}
			\end{align}
			where $(a)$ employs Lemma \ref{lem:sumsum'} and Fact \ref{fact:sum} with $m=|t-\tauit|$ and $(b)$ is implied by Assumption \ref{as:bounded-delay}. 
			Substituting \eqref{eq:Esumsum} with $t=s$ back into \eqref{eq:delayiter} gives
			\begin{align}
				& \E \| \bw^t - \bw^{\tauit} \|_2^2
				\nonumber
				\\
				\leq& 4 K {\lambda} \eta^2 \etabar^2 \frac{\sigma^2}{n}
				+ \frac{2 {\lambda}}{n N} \eta^2 \etabar^2 \sum_{s=(t-{\lambda})_+}^{t-1} \sum_{j=1}^{N} \E \left\| \sum_{k=0}^{K-1} \bG_j^{s,k} \right\|_2^2
				\nonumber
				\\
				& + \frac{2 n(n-1) {\lambda}}{n^2 N^2} \eta^2 \etabar^2 \sum_{s=(t-{\lambda})_+}^{t-1} \E \left\| \sum_{j=1}^{N} \sum_{k=0}^{K-1} \bG_j^{s,k} \right\|_2^2
				\nonumber
				\\
				\leq& 4 K {\lambda} \eta^2 \etabar^2 \frac{\sigma^2}{n}
				+ \frac{2 {\lambda}}{n N} \eta^2 \etabar^2 \sum_{s=(t-{\lambda})_+}^{t-1} \sum_{j=1}^{N} \sum_{k=0}^{K-1} K \E \| \bG_j^{s,k} \|_2^2
				\nonumber
				\\
				& + \frac{2 n(n-1) {\lambda}}{n^2} \eta^2 \etabar^2 \sum_{s=(t-{\lambda})_+}^{t-1} \E \left\| \frac{1}{N} \sum_{j=1}^{N} \sum_{k=0}^{K-1} \bG_j^{s,k} \right\|_2^2
				\nonumber
				\\
				\leq& \hspace{-0.5mm}\frac{4 \sigma^2 K {\lambda} \hspace{-1mm}+\hspace{-1mm} 2 G^2 K^2 {\lambda}^2}{n} \eta^2 \etabar^2
				\hspace{-1.5mm}+\hspace{-1mm} 2 {\lambda} \eta^2 \etabar^2 \hspace{-4mm}\sum_{s=(t-{\lambda})_+}^{t-1} \hspace{-3mm}\E \left\| \hspace{-0.5mm}\frac{1}{N} \hspace{-1mm}\sum_{j=1}^{N}\hspace{-1mm} \sum_{k=0}^{K-1} \hspace{-1mm}\bG_j^{s,k} \right\|_2^2\hspace{-1mm},\hspace{-1mm}
				\nonumber
			\end{align}
			where the second inequality follows from Fact \ref{fact:sum} with $m=K$ and the third inequality is due to Assumption \ref{as:bounded-grad}.
			Combining cases \textsf{i)} and \textsf{ii)}   completes the proof.
		\end{proof}

		\subsubsection{Putting Ingredients Together}
		\begin{proof}
			Using the descent lemma due to the $L$-smoothness by Assumption \ref{as:smooth} and the update formula \eqref{eq:global-iter'}, we have
			\begin{align}
				&\hspace{-1mm}\E [F(\bw^{t+1})] 
				\nonumber
				\\
				\leq& \E [F(\bw^t)] \hspace{-1mm}+\hspace{-0.5mm} \E \langle \nabla F(\bw^t), \bw^{t+1} \hspace{-1mm}-\hspace{-0.5mm} \bw^t \rangle \hspace{-0.8mm}+\hspace{-0.8mm} \frac{L\eta^2}{2} \E \| \bw^{t+1} \hspace{-1mm}-\hspace{-0.5mm} \bw^t \|_2^2
				\nonumber
				\\ 
				=&\E [F(\bw^t)] - \E \langle \nabla F(\bw^t), \eta \bm{g}_{\textnormal{nIID}}^t \rangle + \frac{L\eta^2}{2} \E \| \bm{g}_{\textnormal{nIID}}^t \|_2^2
				\nonumber
				\\
				=& \E [F(\bw^t)] + \frac{L\eta^2}{2} \E \| \bm{g}_{\textnormal{nIID}}^t \|_2^2
				\nonumber
				\\
				& - \eta \E \langle \nabla F(\bw^t), \bm{g}_{\textnormal{nIID}}^t - K \etabar \nabla F(\bw^t) + K \etabar \nabla F(\bw^t) \rangle 
				\nonumber
				\\
				=& \E [F(\bw^t)] - K \eta \etabar \E \| \nabla F(\bw^t) \|_2^2 
				\nonumber
				\\
				& \hspace{-0.8mm}+\hspace{-0.8mm} \eta \E \langle \nabla F(\bw^t),  \etabar K \nabla F(\bw^t) \hspace{-0.5mm}-\hspace{-0.5mm} \bm{g}_{\textnormal{nIID}}^t \rangle \hspace{-0.8mm}+\hspace{-0.8mm} \frac{L\eta^2}{2} \E \| \bm{g}_{\textnormal{nIID}}^t \|_2^2.
				\label{eq:descentlem}
			\end{align}
			Substituting the last two term in the above inequality by employing Lemmas \ref{lem:inner-prod-hetero} and \ref{lem:grad-var'}, we have
			\begin{align}
				&\E [F(\bw^{t+1})] - \E [F(\bw^t)] 
				\nonumber
				\\
				\leq& \hspace{-1mm}-\hspace{-1mm} \frac{K}{2} \eta \etabar \E \| \nabla F(\bw^t) \|_2^2 
				\hspace{-1mm}+\hspace{-1mm} L^2 \eta \etabar \frac{1}{N}\hspace{-1mm} \sum_{i=1}^{N}\hspace{-1mm} \sum_{k=0}^{K-1} \E \| \bw_i^{\tauit,k} \hspace{-1.5mm}-\hspace{-0.5mm} \bw^{\tauit} \|_2^2
				\nonumber
				\\
				&+ \frac{2 \sigma^2 L K + G^2 L K^2}{n} \eta^2 \etabar^2
				+ \frac{L^2 K}{N} \eta \etabar \sum_{i=1}^{N} \E \| \bw^{\tauit} \hspace{-1mm}-\hspace{-0.5mm} \bw^t \|_2^2
				\nonumber
				\\
				&- \left( \frac{1}{2K} \eta \etabar - L \eta^2 \etabar^2 \right) \E \left\| \frac{1}{N} \sum_{i=1}^{N} \sum_{k=0}^{K-1} \bG_i^{t,k} \right\|_2^2.
				\nonumber
			\end{align}
			Substituting $\frac{1}{N} \sum_{i=1}^{N} \sum_{k=0}^{K-1} \E \| \bw_i^{\tauit,k} - \bw^{\tauit} \|_2^2$ using Lemma \ref{lem:local-iter-diff‘}, we have
			\begin{align}
				&~\E [F(\bw^{t+1})] - \E [F(\bw^t)] 
				\nonumber
				\\
				\leq&- \left( \frac{K}{2} \eta \etabar - 16 L^2 K^3 \eta \etabar^3 \right) \E \| \nabla F(\bw^t) \|_2^2 
				\nonumber
				\\
				&+ \frac{2 \sigma^2 L K + G^2 L K^2}{n} \eta^2 \etabar^2
				+ (\sigma^2+8K\nu^2) 2 L^2 K^2 \eta \etabar^3 
				\nonumber
				\\
				&+ \hspace{-1mm}\left(\hspace{-0.5mm} \frac{L^2 K \hspace{-0.5mm}+\hspace{-0.5mm} 16 L^4 K^3}{N} \eta \etabar \hspace{-0.5mm}+\hspace{-0.5mm} \frac{16 L^4 K^3}{N} \eta \etabar^3 \hspace{-0.5mm}\right) \hspace{-0.5mm}\sum_{i=1}^{N} \E \| \bw^{\tauit} \hspace{-1mm}-\hspace{-0.5mm} \bw^t \|_2^2 
				\nonumber
				\\
				& - \left( \frac{1}{2K} \eta \etabar - L \eta^2 \etabar^2 \right) \E \left\| \frac{1}{N} \sum_{i=1}^{N} \sum_{k=0}^{K-1} \bG_i^{t,k} \right\|_2^2
				\nonumber
				\\
				\leq& - \frac{K}{4} \eta \etabar \E \| \nabla F(\bw^t) \|_2^2 
				+ \frac{2 \sigma^2 L K + G^2 L K^2}{n} \eta^2 \etabar^2
				\nonumber
				\\
				&+ \hspace{-0.5mm}(\sigma^2\hspace{-0.5mm}+\hspace{-0.5mm}8K\nu^2) 2 L^2 K^2 \eta \etabar^3 
				\hspace{-0.5mm}+\hspace{-0.5mm} \frac{2 L^2 K}{N} \eta \etabar \sum_{i=1}^{N} \E \| \bw^{\tauit} \hspace{-0.5mm}-\hspace{-0.5mm} \bw^t \|_2^2 
				\nonumber
				\\
				&- \frac{1}{4K} \eta \etabar \E \left\| \frac{1}{N} \sum_{i=1}^{N} \sum_{k=0}^{K-1} \bG_i^{t,k} \right\|_2^2,
				\label{eq:temp''}
			\end{align}
			where the last inequality holds because of condition \eqref{eq:lr'}:
			\begin{align*}
				\etabar \leq \frac{1}{8LK} &\Longleftrightarrow \frac{K}{2} \eta \etabar - 16 L^2 K^3 \eta \etabar^3 \geq \frac{K}{4} \eta \etabar, 
				\\
				\etabar \leq \frac{1}{4LK} &\Longleftrightarrow L^2 K \eta \etabar + 16 L^4 K^3 \eta \etabar^3 \leq 2 L^2 K \eta \etabar,
				\\
				\eta \etabar \leq \frac{1}{2 L K} &\Longleftrightarrow \frac{1}{2K} \eta \etabar - L \eta^2 \etabar^2 \geq \frac{1}{4K} \eta \etabar.
			\end{align*}
			Substituting $\frac{1}{N} \sum_{i=1}^{N} \E \| \bw^{\tauit} - \bw^t \|_2^2$ in \eqref{eq:temp''} using Lemma \ref{lem:global-iter-diff'}, we obtain
			\begin{align}
				&~\E [F(\bw^{t+1})] - \E [F(\bw^t)] 
				\nonumber
				\\
				\leq& - \frac{K}{4} \eta \etabar \E \| \nabla F(\bw^t) \|_2^2 
				+ \frac{2 \sigma^2 L K + L G^2 K^2}{n} \eta^2 \etabar^2
				\nonumber
				\\
				&+ \hspace{-0.5mm}(\sigma^2\hspace{-1mm}+\hspace{-1mm}8K\nu^2) 2 L^2 K^2 \eta \etabar^3
				\hspace{-1mm}+\hspace{-1mm} \frac{8 \sigma^2 L^2 K^2 {\lambda} \hspace{-1mm}+\hspace{-1mm} 4 L^2 G^2 K^3 {\lambda}^2}{n}  \eta^3 \etabar^3
				\nonumber
				\\
				&- \frac{1}{4K} \eta \etabar \E \left\| \frac{1}{N} \sum_{i=1}^{N} \sum_{k=0}^{K-1} \bG_i^{t,k} \right\|_2^2
				\nonumber
				\\
				&+ 4 L^2 K {\lambda} \eta^3 \etabar^3 \sum_{s=(t-{\lambda})_+}^{t-1} \E \left\| \frac{1}{N} \sum_{i=1}^{N} \sum_{k=0}^{K-1} \bG_i^{s,k} \right\|_2^2 
				\nonumber
				\\
				\leq& - \frac{K}{4} \eta \etabar \E \| \nabla F(\bw^t) \|_2^2 
				+ \frac{4 \sigma^2 L K + 2 L G^2 K^2}{n} \eta^2 \etabar^2
				\nonumber
				\\
				& + \hspace{-0.5mm}(\sigma^2\hspace{-0.5mm}+\hspace{-0.5mm}8K\nu^2) 2 L^2 K^2 \eta \etabar^3 
				\hspace{-0.5mm}-\hspace{-0.5mm} \frac{1}{4K} \eta \etabar \E \left\| \frac{1}{N} \sum_{i=1}^{N} \sum_{k=0}^{K-1} \bG_i^{t,k} \right\|_2^2
				\nonumber
				\\
				& + 4 L^2 K {\lambda} \eta^3 \etabar^3 \sum_{s=(t-{\lambda})_+}^{t-1} \E \left\| \frac{1}{N} \sum_{i=1}^{N} \sum_{k=0}^{K-1} \bG_i^{s,k} \right\|_2^2,
				\label{eq:fval-diff'}
			\end{align}
			where the second inequality holds by using condition \eqref{eq:lr'}:
			\begin{align}
				&~\eta \etabar \leq \frac{2 \sigma^2 + G^2 K}{8 \sigma^2 L K \lambda + 4 G^2 L K^2 \lambda^2}
				\nonumber
				\\
				\Leftrightarrow&~
				\frac{2 \sigma^2 L K \hspace{-1mm}+\hspace{-1mm} L G^2 K^2}{n} \eta^2 \etabar^2
				\hspace{-1mm}+\hspace{-1mm} \frac{8 \sigma^2 L^2 K^2 \lambda \hspace{-1mm}+\hspace{-1mm} 4 L^2 G^2 K^3 {\lambda}^2}{n}  \eta^3 \etabar^3 
				\nonumber
				\\
				&~\leq \frac{4 \sigma^2 L K + 2 L G^2 K^2}{n} \eta^2 \etabar^2.
			\end{align}
			By summing inequality \eqref{eq:fval-diff'} over $t = 0, 1, \dots, T-1$, we obtain
			\begin{align}
				&\E [F(\bw^T)] - F[\bw^0] 
				\nonumber
				\\
				\leq& - \hspace{-1mm}\frac{K}{4} \eta \etabar \sum_{t=0}^{T-1} \E \| \nabla F(\bw^t) \|_2^2
				\hspace{-1mm}+\hspace{-1mm} \frac{4 \sigma^2 L K T \hspace{-1mm}+\hspace{-1mm} 2 L G^2 K^2 T}{n} \eta^2 \etabar^2
				\nonumber
				\\
				& + \hspace{-1mm}(\hspace{-0.5mm}\sigma^2\hspace{-1mm}+\hspace{-1mm}8K\nu^2) 2 L^2 K^2 T \eta \etabar^3 
				\hspace{-1mm}-\hspace{-1mm} \frac{1}{4K} \eta \etabar \hspace{-1mm}\sum_{t=0}^{T-1}\hspace{-1mm} \E \hspace{-0.5mm}\left\|\hspace{-0.5mm} \frac{1}{N} \sum_{i=1}^{N} \hspace{-1mm}\sum_{k=0}^{K-1}\hspace{-1mm} \bG_i^{t,k} \hspace{-0.5mm}\right\|_2^2 
				\nonumber
				\\
				&+ 4 L^2 K {\lambda} \eta^3 \etabar^3 \sum_{t=0}^{T-1} \sum_{s=(t-{\lambda})_+}^{t-1} \underbrace{\E \left\| \frac{1}{N} \sum_{i=1}^{N} \sum_{k=0}^{K-1} \bG_i^{s,k} \right\|_2^2}_{Y^s}
				\nonumber
				\\
				\leq& - \frac{K}{4} \eta \etabar \sum_{t=0}^{T-1} \E \| \nabla F(\bw^t) \|_2^2
				+ \frac{4 \sigma^2 L K T + 2 L G^2 K^2 T}{n} \eta^2 \etabar^2 
				\nonumber
				\\
				&+ (\sigma^2+8K\nu^2) 2 L^2 K^2 T \eta \etabar^3 
				\nonumber
				\\
				&- \left( \frac{1}{4K} \eta \etabar  - 4 L^2 K {\lambda}^2 \eta^3 \etabar^3 \right) \sum_{t=0}^{T-1} \E \left\| \frac{1}{N} \sum_{i=1}^{N} \sum_{k=0}^{K-1} \bG_i^{t,k} \right\|_2^2,
				\nonumber
				\\
				\leq& - \frac{K}{4} \eta \etabar \sum_{t=0}^{T-1} \E \| \nabla F(\bw^t) \|_2^2
				+ \frac{4 \sigma^2 L K T + 2 L G^2 K^2 T}{n} \eta^2 \etabar^2
				\nonumber
				\\
				&+ (\sigma^2+8K\nu^2) 2 L^2 K^2 T \eta \etabar^3,
				\nonumber
			\end{align}
			where the second inequality uses the fact that $\sum_{t=0}^{T-1} \sum_{s=(t-{\lambda})_+}^{t-1} Y^s \leq \sum_{t=0}^{T-2} {\lambda} Y^t \leq {\lambda} \sum_{t=0}^{T-1} Y^t$ and the last inequality holds because of condition \eqref{eq:lr}: 
			$$
			\eta \etabar \leq \frac{1}{4 L K {\lambda}} \Longleftrightarrow \frac{1}{4K} \eta \etabar - 4 L^2 K {\lambda}^2 \eta^3 \etabar^3 \geq 0.
			$$
			Rearranging and using the fact that $F^* \leq \E [F(\bw^T)] $, we have
			\begin{align}
				&\frac{1}{T} \sum_{t=0}^{T-1} \E \| \nabla F(\bw^t) \|_2^2
				\nonumber
				\\
				=& \frac{4(F(\bw^0)\hspace{-1mm}-\hspace{-1mm}F^*)}{KT\eta\etabar}
				\hspace{-1mm}+\hspace{-1mm} \frac{16 \sigma^2 L \hspace{-1mm}+\hspace{-1mm} 8 L G^2 K}{n} \eta \etabar 
				\hspace{-1mm}+\hspace{-1mm} (\sigma^2\hspace{-1mm}+\hspace{-1mm}8K\nu^2) 8 L^2 K \etabar^2
				\nonumber
				\\
				=&\hspace{-0.5mm} \sqrt{\hspace{-0.5mm}\frac{256 (\hspace{-0.5mm}F(\hspace{-0.5mm}\bw^0\hspace{-0.5mm})\hspace{-1mm}-\hspace{-1mm}F^*\hspace{-0.5mm})(\hspace{-0.5mm}16 \sigma^2 L \hspace{-1mm}+\hspace{-1mm} 8 L G^2 K\hspace{-0.5mm})}{n K T}} 
				\hspace{-1mm}+\hspace{-1mm} \frac{(\sigma^2 \hspace{-1mm}+\hspace{-1mm} 8K\nu^2) 8 L^2}{(\hspace{-0.5mm}4 \sigma^2 L \hspace{-1mm}+\hspace{-1mm} 2 L K G^2\hspace{-0.5mm}) K T},
				\label{eq:summing}
			\end{align}
			where the second equality follows by substituting $\eta = \sqrt{4 n K (F(\bw^0)-F^*)}$ and $\etabar = 1/\sqrt{(4 \sigma^2 L + 2 L K G^2)T}K$. 
			This completes the proof of Theorem \ref{thm:hetero}.
		\end{proof}

		\subsection{Proof of Theorem \ref{thm:iid}}\label{sec:proof1}
		\subsubsection{Technical Lemmas}
		\begin{lemma}\label{lem:inner-prod}
			Suppose that Assumptions \ref{as:smooth} and \ref{as:unbias} hold. Then, it hold for all $t \geq 0$ that
			\begin{align*}
				&~\E \langle \nabla F(\bw^t), \etabar K \nabla F(\bw^t) - \bm{g}_{\textnormal{IID}}^t \rangle
				\nonumber
				\\
				\leq&~ \frac{K}{2} \etabar \E \| \nabla F(\bw^t) \|_2^2 
				+ \frac{L^2}{n} \etabar \sum_{i \in {I}_t} \sum_{k=0}^{K-1} \E \| \bw_i^{\tauit,k} - \bw^{\tauit} \|_2^2
				\nonumber
				\\
				&+ \hspace{-0.5mm}\frac{L^2 K}{n} \etabar \sum_{i \in {I}_t} \hspace{-0.5mm}\E \| \bw^{\tauit} \hspace{-0.5mm}-\hspace{-0.5mm} \bw^t \|_2^2 
				\hspace{-0.5mm}-\hspace{-0.5mm} \frac{1}{2K} \etabar \E \left\| \frac{1}{n} \hspace{-0.5mm}\sum_{i \in {I}_t}\hspace{-0.5mm} \sum_{k=0}^{K-1} \hspace{-0.5mm}\bG_i^{t,k} \right\|_2^2\hspace{-0.5mm}.
			\end{align*}
		\end{lemma}
		\begin{proof}
			Using the local iteration formula of DeFedAvg-IID, we have $\bm{\Delta}_i^t = \bw_i^{\tauit,0} - \bw_i^{\tauit,K} = \sum_{k=0}^{K-1} \etabar \bg_i^{t,k}$, which implies that
			\begin{align}
				\bm{g}_{\textnormal{IID}}^t \coloneqq \frac{1}{|I_t|} \sum_{i \in I_t} \bm{\Delta}_i^t = \frac{\etabar}{n} \sum_{i \in {I}_t} \sum_{k=0}^{K-1} \bg_i^{t,k}.
				\label{eq:gt}
			\end{align}
			Therefore, we have 
			\begin{align}
				& \E \langle \nabla F(\bw^t),  \etabar K \nabla F(\bw^t) - \bm{g}_{\textnormal{IID}}^t \rangle
				\nonumber
				\\
				=& \etabar \E \hspace{-0.5mm}\left\langle\hspace{-1.5mm} \sqrt{K} \nabla F(\bw^t), \hspace{-0.5mm}\frac{1}{\sqrt{K} n}\hspace{-0.5mm} \sum_{i \in {I}_t} \hspace{-1mm} \left( \hspace{-1mm}K \nabla F(\bw^t) \hspace{-0.5mm}-\hspace{-1mm} \sum_{k=0}^{K-1}\hspace{-0.5mm} \bg_i^{t,k} \hspace{-0.5mm}\right)\hspace{-1.5mm} \right\rangle\hspace{-0.5mm}.\hspace{-0.5mm}
				\nonumber
			\end{align}
			The remaining can be proved using similar arguments as the proof of Lemma \ref{lem:inner-prod-hetero}, which is omitted for brevity.
		\end{proof}
		
		\begin{lemma}\label{lem:sumsum}
			Suppose that Assumptions \ref{as:unbias}, \ref{as:bounded-var}, and \ref{as:bounded-delay} hold. Then, it holds for all $t \geq 0$ that
			\begin{align}
				\E \left\| \sum_{j \in {I}_t} \sum_{k=0}^{K-1} \left( \bg_i^{t,k} - \bG_i^{t,k} \right) \right\|_2^2
				&\leq n K \sigma^2,
				\label{eq:lema} 
			\end{align}
			Besides, it holds for all $t \geq 1$ and $i \in {I}_t$ that
			\begin{align}
				\E \left\| \sum_{s=\tauit}^{t-1} \sum_{j \in {I}_s} \sum_{k=0}^{K-1} \left( \bg_j^{s,k} - \bG_j^{s,k} \right) \right\|_2^2 
				&\leq n K {\lambda} \sigma^2.
				\label{eq:lemb}
			\end{align}
		\end{lemma}
		\begin{proof}
			The proof is similar to that of Lemma \ref{lem:sumsum'}, while there is no randomness in $I_t$. Following analogous arguments in the proof of Lemma \ref{lem:sumsum'}, we can obtain 
			\begin{align}
				&\E \left\| \sum_{j \in {I}_t} \sum_{k=0}^{K-1}  \left( \bg_i^{t,k} - \bG_i^{t,k} \right) \right\|_2^2
				\hspace{-0.5mm}=\hspace{-0.5mm}\sum_{j \in {I}_t} \E \left\| \sum_{k=0}^{K-1} \left( \bg_j^{t,k} - \bG_j^{t,k} \right)  \right\|_2^2
				\nonumber
				\\
				&= \sum_{j \in {I}_t} \sum_{k=0}^{K-1} \E \left\| \bg_j^{t,k} - \bG_j^{t,k} \right\|_2^2 \leq n K \sigma^2,
				\nonumber
			\end{align}
			and thus
			\begin{align}
				&\E \left\| \sum_{s=\tauit}^{t-1} \sum_{j \in {I}_s} \sum_{k=0}^{K-1} \left( \bg_j^{s,k} - \bG_j^{s,k} \right) \right\|_2^2 
				\nonumber
				\\
				&\hspace{-1mm}=\hspace{-2mm}\sum_{s=\tauit}^{t-1}\hspace{-1.5mm} \E \hspace{-0.5mm}\left\| \sum_{j \in {I}_s}\hspace{-1mm} \sum_{k=0}^{K-1}\hspace{-1.5mm} \left(\hspace{-0.5mm} \bg_j^{s,k} \hspace{-1.5mm}-\hspace{-0.5mm} \bG_j^{s,k} \hspace{-0.5mm}\right) \hspace{-0.5mm}\right\|_2^2 
				\hspace{-0.5mm}\leq\hspace{-0.5mm} \sum_{s=(t-{\lambda})_+}^{t-1}\hspace{-3mm} n K \sigma^2
				\hspace{-0.5mm}\leq\hspace{-0.5mm} n K {\lambda} \sigma^2,
				\nonumber
			\end{align}
			which complete the proof.
		\end{proof}
		
		\begin{lemma}\label{lem:grad-var}
			Suppose that Assumptions \ref{as:unbias}, \ref{as:bounded-var}, and \ref{as:bounded-delay} hold. Then, it hold for all $t \geq 0$ that
			\begin{align}
				\E \| \bm{g}_{\textnormal{IID}}^t \|_2^2 \leq 2 K \etabar^2 \frac{\sigma^2}{n} 
				+ 2\etabar^2 \E \left\| \frac{1}{n} \sum_{i \in {I}_t} \sum_{k=0}^{K-1} \bG_i^{t,k} \right\|_2^2.
			\end{align}
			\begin{proof}
				In view of \eqref{eq:gt}, we have
				\begin{align}
					& \E \| \bm{g}_{\textnormal{IID}}^t \|_2^2 
					= \E \left\| \frac{1}{n} \sum_{i \in {I}_t} \sum_{k=0}^{K-1} \etabar \bg_i^{t,k} \right\|_2^2 
					\nonumber
					\\
					&\hspace{-0.5mm}\stackrel{(a)}{\leq} \hspace{-0.5mm} 2 \E\hspace{-0.5mm} \left\| \hspace{-0.5mm}\frac{1}{n} \hspace{-0.5mm}\sum_{i \in {I}_t}\hspace{-0.5mm} \sum_{k=0}^{K-1} \hspace{-0.5mm}\etabar \hspace{-0.5mm}\left(\hspace{-0.5mm} \bg_i^{t,k} \hspace{-0.5mm}-\hspace{-1mm} \bG_i^{t,k} \right) \right\|_2^2
					\hspace{-1mm}+\hspace{-0.5mm} 2 \E \left\| \frac{1}{n} \sum_{i \in {I}_t} \sum_{k=0}^{K-1} \etabar \bG_i^{t,k} \right\|_2^2
					\nonumber
					\\
					&\stackrel{(b)}{\leq} 2 K \etabar^2 \frac{\sigma^2}{n} 
					+ 2\etabar^2 \E \left\| \frac{1}{n} \sum_{i \in {I}_t} \sum_{k=0}^{K-1} \bG_i^{t,k} \right\|_2^2.
					\label{eq:Delta-temp}
				\end{align}
				where $(a)$ uses Fact \ref{fact:sum} with $m=2$, $(b)$ follows from \eqref{eq:lema}.
			\end{proof}
		\end{lemma}
		
		\begin{lemma}\label{lem:local-iter-diff}
			Suppose that Assumption \ref{as:smooth}, \ref{as:unbias}, and \ref{as:bounded-var} hold and $\etabar \leq \frac{1}{2\sqrt{2}LK}$. Then, it holds for all $t \geq 0$ that
			\begin{align*}
				&\frac{1}{n} \sum_{i \in {I}_t} \sum_{k=0}^{K-1} \E \| \bw_i^{\tauit,k} - \bw^{\tauit} \|_2^2
				\leq 2 K^2 \etabar^2 \sigma^2 
				\nonumber
				\\
				&+\hspace{-0.5mm} 12 K^3 \etabar^2  \E \| \nabla F(\bw^t) \|_2^2 
				\hspace{-0.5mm}+\hspace{-0.5mm} \frac{12 L^2 K^3}{n} \etabar^2 \sum_{i \in {I}_t} \E \| \bw^{\tauit} \hspace{-0.5mm}-\hspace{-0.5mm} \bw^t \|_2^2.
			\end{align*}
			\begin{proof}
				The proof is similar to that of Lemma \ref{lem:local-iter-diff‘}. We simply need to change \eqref{eq:lem:local-iter-diff‘1} and \eqref{eq:recur'} as 
				\begin{align}
					&\E \| \bw_i^{\tauit,k} - \bw^{\tauit} \|_2^2
					\nonumber
					\\
					=& \E \Big\| \bw_i^{\tauit,k\hspace{-0.5mm}-\hspace{-0.5mm}1} \hspace{-1.5mm}-\hspace{-1mm} \bw^{\tauit} \hspace{-1mm}-\hspace{-1mm} \etabar\Big(\hspace{-1mm} \bG_i^{t,k\hspace{-0.5mm}-\hspace{-0.5mm}1} \hspace{-1mm}-\hspace{-1mm} \nabla\hspace{-0.5mm}F(\bw^{\tauit}) \hspace{-1mm}+\hspace{-1mm} \nabla \hspace{-0.5mm} F(\bw^{\tauit}) 
					\nonumber
					\\ 
					& - \nabla F(\bw^t) + \nabla F(\bw^t) \Big) \Big\|_2^2 
					+ \etabar^2 \E \| \bg_i^{t,k-1} - \bG_i^{t,k-1} \|_2^2
					\nonumber
				\end{align}
				and
				\begin{align}
					&\E \| \bw_i^{\tauit,k} \hspace{-1.5mm}-\hspace{-0.5mm} \bw^{\tauit} \|_2^2
					\hspace{-0.5mm}\leq\hspace{-1mm} \left(\hspace{-1mm} 1 \hspace{-0.5mm}+\hspace{-0.5mm} \frac{1}{K\hspace{-0.5mm}-\hspace{-0.5mm}1}\hspace{-0.5mm}\right) \hspace{-0.5mm}\E \| \bw_i^{\tauit,k-1} \hspace{-1.5mm}-\hspace{-0.5mm} \bw^{\tauit} \|_2^2
					\nonumber
					\\
					&+ 6 L^2 K \etabar^2 \E \| \bw^{\tauit} \hspace{-0.5mm}-\hspace{-1mm} \bw^t \|_2^2
					\hspace{-0.5mm}+\hspace{-0.5mm} 6 K \etabar^2  \E \| \nabla F(\bw^t) \|_2^2 
					\hspace{-0.5mm}+\hspace{-0.5mm} \etabar^2 \sigma^2\hspace{-0.5mm},\hspace{-0.5mm}
					\nonumber
				\end{align}
				respectively. The remaining proof is omitted for brevity.
			\end{proof}
		\end{lemma}
		
		\begin{lemma}\label{lem:global-iter-diff}
			Suppose that Assumptions \ref{as:unbias}, \ref{as:bounded-var}, and \ref{as:bounded-delay} hold. Then, it holds for all $t \geq 0$ and $i \in {I}_t$ that
			\begin{align}
				&\E \| \bw^t - \bw^{\tauit} \|_2^2 
				\leq 2 K {\lambda} \eta^2 \etabar^2 \frac{\sigma^2}{n}
				\nonumber
				\\
				&+ 2 {\lambda} \eta^2 \etabar^2 \sum_{s=(t-{\lambda})_+}^{t-1} \E \left\| \frac{1}{n} \sum_{j \in {I}_s} \sum_{k=0}^{K-1} \bG_j^{s,k} \right\|_2^2.
				\nonumber
			\end{align}
		\end{lemma}
		\begin{proof}
			The proof is a simple adjustment of \eqref{eq:delayiter} by employing  $\bw^{s+1} = \bw^s - \frac{\eta}{|I_s|} \sum_{j \in I_s} \etabar \sum_{k=0}^{K-1} \bg_j^{s,k}$ and Lemma \ref{lem:sumsum}.
			The details are omitted for brevity.
		\end{proof}
		
		\subsubsection{Putting Ingredients Together}
		\begin{proof}
			\begin{align}
				&\E [F(\bw^{t+1})] \hspace{-1mm}-\hspace{-0.5mm} \E [F(\bw^t)] 
				\hspace{-1mm}\leq\hspace{-1mm} - \frac{K}{4} \eta \etabar \E \| \nabla F(\bw^t) \|_2^2 
				\hspace{-1mm}+\hspace{-1mm} 2 L K \eta^2 \etabar^2 \frac{\sigma^2}{n}
				\nonumber
				\\
				&\quad + 2 L^2 K^2 \eta \etabar^3 \sigma^2
				- \frac{1}{4K} \eta \etabar \E \left\| \frac{1}{n} \sum_{i \in {I}_t} \sum_{k=0}^{K-1} \bG_i^{t,k} \right\|_2^2 
				\nonumber
				\\
				&\quad + 4 L^2 K {\lambda} \eta^3 \etabar^3 \sum_{s=(t-{\lambda})_+}^{t-1} \E \left\| \frac{1}{n} \sum_{i \in {I}_s} \sum_{k=0}^{K-1} \bG_i^{s,k} \right\|_2^2,
				\label{eq:fval-diff}
			\end{align}
			Further using similar arguments for proving \eqref{eq:summing} and the learning rate conditions \eqref{eq:lr}, we have
			\begin{align}
				\frac{1}{T} \hspace{-1mm}\sum_{t=0}^{T-1}\hspace{-0.5mm} \E \| \hspace{-0.5mm}\nabla \hspace{-0.5mm}F(\bw^t) \|_2^2
				&\hspace{-1mm}=\hspace{-1mm} \frac{4(\hspace{-0.5mm}F(\hspace{-0.5mm}\bw^0\hspace{-0.5mm})\hspace{-0.5mm}-\hspace{-0.5mm}F^*\hspace{-0.5mm})}{KT\eta\etabar}
				\hspace{-0.5mm}+\hspace{-0.5mm} 8 L \eta \etabar \frac{\sigma^2}{n}
				\hspace{-0.5mm}+\hspace{-0.5mm} 8 L^2 K \etabar^2 \sigma^2
				\nonumber
				\\
				&= \sqrt{\frac{128 L(F(\bw^0)-F^*)}{n K T}} + \frac{8L}{K T},
				\nonumber
			\end{align}
			where the second equality holds since $\eta = \sqrt{(F(\bw^0)-F^*) n K / 2}$ and $\etabar = 1/(\sqrt{\sigma^2 LT}K)$. 
			This completes the proof of Theorem \ref{thm:iid}.
		\end{proof}

		\section{Values of Global and Local Learning Rates}
		We provide the best-tuned values of the global and local learning rates used in our experiments, as shown in Table \ref{tab:learning_rate}.
		
		\begin{table*}[t]
			\centering
			\resizebox{\textwidth}{!}{
				\renewcommand\arraystretch{1.5}
				\begin{tabular}{cccccccccccc}
					\hline \hline
					\multirow{2}{*}{} & \multirow{2}{*}{Alogorithms}        & \multicolumn{5}{c}{FashionMNIST}                                                          & \multicolumn{5}{c}{CIFAR-10}                                                              \\ \cmidrule(lr){3-7} \cmidrule(lr){8-12}
					&                                    & $n=5$ & $n=10$                 & $n=20$                & $n=40$                 & $n=80$                 & $n=5$ & $n=10$                 & $n=20$                 & $n=40$                 & $n=80$                 \\ \hline \specialrule{0em}{1.0pt}{1.0pt}
					\multirow{8}{*}{IID}      & \multirow{2}{*}{FedAvg}         &\multirow{2}{*}{N/A} &      ${\eta}=1.00$                &       ${\eta}=1.00$                &         ${\eta}=1.00$             &     ${\eta}=1.00$        &\multirow{2}{*}{N/A}         &          ${\eta}=1.00$                &    ${\eta}=1.00$                   &       ${\eta}=1.00$                 &          ${\eta}=1.00$              \\ 
					&   &                                     & \multicolumn{1}{c}{$\bar{\eta}=0.10$} & \multicolumn{1}{c}{$\bar{\eta}=0.01$} & \multicolumn{1}{c}{$\bar{\eta}=0.01$} & \multicolumn{1}{c}{$\bar{\eta}=0.01$} & & \multicolumn{1}{c}{$\bar{\eta}=0.01$} &  \multicolumn{1}{c}{$\bar{\eta}=0.01$} & \multicolumn{1}{c}{$\bar{\eta}=0.01$} & \multicolumn{1}{c}{$\bar{\eta}=0.01$} \\ 
					& \multirow{2}{*}{FAVANO}  &          \multirow{2}{*}{N/A}  &       \multirow{2}{*}{$\bar{\eta}=0.10$}               &         \multirow{2}{*}{$\bar{\eta}=0.10$}               &     \multirow{2}{*}{$\bar{\eta}=0.10$}                   &          \multirow{2}{*}{$\bar{\eta}=0.10$}      & \multirow{2}{*}{N/A}       &     \multirow{2}{*}{$\bar{\eta}=0.05$}                &        \multirow{2}{*}{$\bar{\eta}=0.05$}              &    \multirow{2}{*}{$\bar{\eta}=0.05$}                  &      \multirow{2}{*}{$\bar{\eta}=0.05$}                \\
					&                                     &  &  & &  &  &  &  &  \\
					& \multirow{2}{*}{AsySG}  &          \multirow{2}{*}{N/A}  &       \multirow{2}{*}{$\eta=0.10$}               &         \multirow{2}{*}{$\eta=0.10$}               &     \multirow{2}{*}{$\eta=0.10$}                   &          \multirow{2}{*}{$\eta=0.10$}      & \multirow{2}{*}{N/A}       &     \multirow{2}{*}{$\eta=0.10$}                &        \multirow{2}{*}{$\eta=0.10$}              &    \multirow{2}{*}{$\eta=0.10$}                  &      \multirow{2}{*}{$\eta=0.10$}                \\
					&                                     &  &  & &  &  &  &  &  \\
					& \multirow{2}{*}{DeFedAvg-IID}         &      $\eta=0.10$ &      $\eta=0.10$                &         $\eta=0.10$             &    $\eta=0.10$                  &   $\eta=1.00$                   & $\eta=0.10$  &       $\eta=1.00$                    &         $\eta=0.10$              &    $\eta=0.10$                   &          $\eta=1.00$             \\
					&                                     & \multicolumn{1}{c}{$\bar{\eta}=0.05$}& \multicolumn{1}{c}{$\bar{\eta}=0.05$} & \multicolumn{1}{c}{$\bar{\eta}=0.05$} & \multicolumn{1}{l}{$\bar{\eta}=0.05$} & \multicolumn{1}{c}{$\bar{\eta}=0.10$}  & \multicolumn{1}{l}{$\bar{\eta}=0.05$} & \multicolumn{1}{l}{$\bar{\eta}=0.05$} & \multicolumn{1}{c}{$\bar{\eta}=0.01$} & \multicolumn{1}{c}{$\bar{\eta}=0.05$} & \multicolumn{1}{c}{$\bar{\eta}=0.01$}  \\ \hline
					\multirow{8}{*}{non-IID}  & \multirow{2}{*}{FedAvg} &  \multirow{2}{*}{N/A} &        $\eta=0.10$            &        $\eta=0.10$              &           $\eta=1.00$           &       $\eta=0.10$               &  \multirow{2}{*}{N/A}  &    $\eta=1.00$              &       $\eta=1.00$               &            $\eta=0.10$          &       $\eta=0.10$               \\ &
					&                                     & \multicolumn{1}{c}{$\bar{\eta}=0.05$} & \multicolumn{1}{c}{$\bar{\eta}=0.05$} & \multicolumn{1}{c}{$\bar{\eta}=0.01$} & \multicolumn{1}{c}{$\bar{\eta}=0.10$} &  &\multicolumn{1}{c}{$\bar{\eta}=0.01$} & \multicolumn{1}{c}{$\bar{\eta}=0.01$} & \multicolumn{1}{c}{$\bar{\eta}=0.05$} & \multicolumn{1}{c}{$\bar{\eta}=0.05$} \\
					& \multirow{2}{*}{FAVANO}  &          \multirow{2}{*}{N/A}  &       \multirow{2}{*}{$\bar{\eta}=0.05$}               &         \multirow{2}{*}{$\bar{\eta}=0.05$}               &     \multirow{2}{*}{$\bar{\eta}=0.05$}                   &          \multirow{2}{*}{$\bar{\eta}=0.05$}      & \multirow{2}{*}{N/A}       &     \multirow{2}{*}{$\bar{\eta}=0.05$}                &        \multirow{2}{*}{$\bar{\eta}=0.05$}              &    \multirow{2}{*}{$\bar{\eta}=0.05$}                  &      \multirow{2}{*}{$\bar{\eta}=0.05$}                \\
					&                                     &  &  & &  &  &  &  &  \\
					& \multirow{2}{*}{FedBuff}           &\multirow{2}{*}{N/A} &     $\eta=0.10$                 &        $\eta=0.10$              &  $\eta=0.10$                    &   $\eta=0.10$                   &  \multirow{2}{*}{N/A} &             $\eta=0.10$        &    $\eta=0.10$                  &            $\eta=0.10$          &     $\eta=0.10$                 \\
					&       &                              & \multicolumn{1}{c}{$\bar{\eta}=0.005$} & \multicolumn{1}{c}{$\bar{\eta}=0.01$} & \multicolumn{1}{c}{$\bar{\eta}=0.01$} & \multicolumn{1}{c}{$\bar{\eta}=0.05$} & &\multicolumn{1}{c}{$\bar{\eta}=0.01$} & \multicolumn{1}{c}{$\bar{\eta}=0.01$} & \multicolumn{1}{c}{$\bar{\eta}=0.01$} & \multicolumn{1}{c}{$\bar{\eta}=0.05$} \\
					& \multirow{2}{*}{DeFedAvg-nIID}      & $\eta=0.10$  &        $\eta=0.10$              &         $\eta=0.10$             &  $\eta=0.10$                    &           $\eta=0.10$           &  $\eta=0.10$ & ${\eta}=0.10$                    &           ${\eta}=0.10$           &     ${\eta}=0.10$                 &      ${\eta}=0.10$                \\
					&                                   &\multicolumn{1}{c}{$\bar{\eta}=0.05$} & \multicolumn{1}{c}{$\bar{\eta}=0.05$} & \multicolumn{1}{c}{$\bar{\eta}=0.05$} & \multicolumn{1}{l}{$\bar{\eta}=0.05$} & \multicolumn{1}{l}{$\bar{\eta}=0.05$ } & \multicolumn{1}{l}{$\bar{\eta}=0.05$} & \multicolumn{1}{l}{$\bar{\eta}=0.05$} & \multicolumn{1}{l}{$\bar{\eta}=0.10$} & \multicolumn{1}{l}{$\bar{\eta}=0.05$} & \multicolumn{1}{l}{$\bar{\eta}=0.05$} \\ \hline \hline
			\end{tabular}}
			\caption{The best-tuned global learning rate $\eta$ and local learning rate $\bar{\eta}$ for generating the results in Table \ref{tab:main_experiments}.}
			\label{tab:learning_rate}
		\end{table*}
		
	\end{document}